\documentclass[journal,twocolumn]{IEEEtran}
\usepackage{braket}
\usepackage{array}
\usepackage[caption=false]{subfig}
\usepackage{pgfplots}

\usepackage[caption=false]{subfig}
\usepackage{subfig}
\usepackage[normalem]{ulem}
\usepackage[hidelinks]{hyperref}
\usepackage{url}            
\usepackage{booktabs}       
\usepackage{amsfonts}       
\usepackage{nicefrac}       
\usepackage{microtype}      
\usepackage{yfonts,dsfont}
\usepackage{setspace}
\usepackage{verbatim}

\usepackage{subfig}
\usepackage{cite}
\usepackage{pifont}
\usepackage{lipsum}
\usepackage{mathtools}
\usepackage{mathcomp}
\usepackage{euscript,bbm,amstext,wasysym,arydshln,framed}
\usepackage{upgreek}
\usepackage{tikz}
\usepackage{amssymb}
\usetikzlibrary{positioning,shapes,arrows,calc,matrix}

\usetikzlibrary{external}

\newtheorem{remark}{Remark}

\newtheorem{example}{Example}

\newcommand{\enma}[1]   {\ensuremath{#1}}

\newcommand{\beq}{\begin{equation}}
\newcommand{\eeq}{\end{equation}}
\newcommand{\bseq}{\begin{subequations}}
\newcommand{\eseq}{\end{subequations}}
\newcommand{\beqn}{\begin{eqnarray}}
\newcommand{\eeqn}{\end{eqnarray}}
\newcommand{\ba}{\begin{array}}
\newcommand{\ea}{\end{array}}
\newcommand{\bct}{\begin{center}}
\newcommand{\ect}{\end{center}}
\newcommand{\btmz}{\begin{itemize}}
\newcommand{\etmz}{\end{itemize}}
\newcommand{\benum}{\begin{enumerate}}
\newcommand{\eenum}{\end{enumerate}}

\newcommand{\norm}[1]{\| #1 \|}                 

\newcommand{\diag}      {\enma{\mathrm{diag}}}

\newcommand{\trace}     {\enma{\mathrm{trace}}}

\newcommand{\eig}       {\enma{\mathrm{eig}}}

\newcommand{\matbegin}{
        \left[
}
\newcommand{\matend}{
        \right]
}

\newcommand{\tbo}[2]{
  \matbegin \begin{array}{c}
       #1 \\ #2
       \end{array} \matend }

\newcommand{\obt}[2]{
  \matbegin \begin{array}{cc}
       #1 & #2
       \end{array} \matend }

\newcommand{\tbt}[4]{
  \matbegin \begin{array}{cc}
       #1 & #2 \\ #3 & #4
       \end{array} \matend }

\newcommand{\be}{\begin{equation}}
\newcommand{\ee}{\end{equation}}

\newcommand{\cplxs}{ C\kern -.35em \rule{0.03 em}{.7 ex}~   }

\def\complex{\hbox{C\kern -.45em \rule{0.03 em}{1.5 ex}}~}

\newcommand{\bi}{\begin{itemize}}
\newcommand{\ei}{\end{itemize}}

\renewcommand{\baselinestretch}{.99} 

\newtheorem{mythm}{Theorem}
\newtheorem{myprop}{Proposition}
\newtheorem{mylem}{Lemma}
\newtheorem{myrem}{Remark}

\newtheorem{mycor}{Corollary}

\newcommand{\R}{\mathbb{R}}

\newcommand{\cH}{{\cal H}}

\newcommand{\non}{\nonumber}

\newcommand{\mrd}{\mathrm{d}}
\newcommand{\mre}{\mathrm{e}}

\newcommand{\tc}{\textcolor}
\newcommand{\tcb}{\textcolor{black}}

\newcommand{\vsp}{\vspace*{0.15cm}}

\newcommand{\DefinedAs}[0]{\mathrel{\mathop:}=}
\newcommand{\AsDefined}[0]{=\mathrel{\mathop:}}
\DeclareMathOperator*{\argmin}{argmin}

\DeclareMathOperator*{\minimize}{minimize}

\DeclareMathOperator*{\rank}{rank}

\usepackage{tikz}
\usetikzlibrary{3d}
\usetikzlibrary{shapes.symbols,positioning,shapes,arrows,calc,matrix}
\usetikzlibrary{decorations.markings}
\usetikzlibrary{arrows,decorations.pathmorphing}
\usetikzlibrary{positioning,shapes,arrows,calc,matrix}
\usepackage{pgfplots}
\usepgfplotslibrary{fillbetween}
\usepackage{circuitikz}
\usetikzlibrary{patterns,decorations.pathmorphing,quotes,angles}
\usepackage{sidecap}
 
\usepackage{graphicx}

\usepackage[english]{babel}
\usepackage{blindtext}

\definecolor{hmaroon}{RGB}{128,0,0}

\IEEEoverridecommandlockouts                              
\begin{document}
	
	\title{\LARGE \bf Stability properties of gradient flow dynamics for \\[0.1 cm] the symmetric low-rank matrix factorization problem
	}
	
	\author{Hesameddin Mohammadi, Mohammad Tinati, Stephen Tu, Mahdi Soltanolkotabi, and
			Mihailo R.\ Jovanovi\'c}
		
	\maketitle
	\begin{abstract}
	The symmetric low-rank matrix factorization serves as a building block in many learning tasks, including matrix recovery and training of neural networks. However, despite a flurry of recent research, the dynamics of its training via non-convex factorized gradient-descent-type methods is not fully understood especially in the over-parameterized regime where the fitted rank is higher than the true rank of the target matrix. To overcome this challenge, we characterize equilibrium points of the gradient flow dynamics and examine their local and global stability properties. To facilitate a precise global analysis, we introduce a nonlinear change of variables that brings the dynamics into a cascade connection of three subsystems whose structure is simpler than the structure of the original system. We demonstrate that the Schur complement to a principal eigenspace of the target matrix is governed by an autonomous system that is decoupled from the rest of the dynamics. In the over-parameterized regime, we show that this Schur complement vanishes at an $O(1/t)$ rate, thereby capturing the slow dynamics that arises from excess parameters. We utilize a Lyapunov-based approach to establish exponential convergence of the other two subsystems. By decoupling the fast and slow parts of the dynamics, we offer new insight into the shape of the trajectories associated with local search algorithms and provide a complete characterization of the equilibrium points and their \tcb{global} stability properties. Such an analysis via nonlinear control techniques may prove useful in several related over-parameterized problems.
	\end{abstract}
	 	
		\vspace*{-0.4cm}
	\section{Introduction}
	
	Training of massive over-parameterized models that contain many more parameters than training data via first-order algorithms such as (stochastic) gradient descent is a work-horse of modern machine learning. Since these learning problems are typically non-convex, it is not clear if first-order methods converge to a global optimum. Furthermore, because of over-parameterization, even when global convergence does occur, the computed solution may not have good generalization properties. 
		
	Recent years have witnessed a surge of research efforts aimed at demystifying the complexities of optimization and generalization in such problems. A number of papers demonstrate local convergence of search techniques starting from carefully designed spectral initializations \cite{wirtinger_flow,truncated_wirtinger_flow,chen_implicit_regularization,tu2015low,ling_blind_deconv,ling_demixing,netrapalli_alternating,waldspurger_alternating, stoger2024non}. Another set of publications focuses on showing that the optimization landscape is favorable, i.e., that all local minima are global, with saddle points exhibiting strict negative curvature~\cite{sun2015nonconvex}. In these, algorithms like trust region methods, cubic regularization~\cite{nesterov2006cubic, nocedal2006trust}, and stochastic gradient techniques~\cite{jin2017escape,ge2015escaping, raginsky2017non,zhang2017hitting} are employed to find global optima. There is a growing consensus that a deeper, more detailed analysis of gradient descent trajectories is needed beyond merely examining the optimization landscape~\cite{arora2019implicit} or specialized initialization techniques. While a number of papers have started to address this challenge, they \tcb{typically} require very intricate and specialized proofs~\cite{chen_global_convergence,stoger2021small,soltanolkotabi2023implicit,ye2021global} that often are not able to fully characterize the trajectory from various initialization points or precisely capture the speed of convergence. 

 
In this letter, we use a control-theoretic approach to examine the behavior of gradient flow dynamics for solving the symmetric low-rank matrix factorization problem; the rank-one case was addressed in~\cite{mohrazjovACC18}. This problem serves as a building block for several non-convex learning tasks, including matrix completion~\cite{ge2016matrix,kawaguchi2016deep}, and deep neural networks~\cite{arora2019fine}, and it has received significant attention in the literature~\cite{oja1982simplified,oja1989neural,helmke2012optimization}. Although the optimization landscape of the matrix factorization is benign in the sense that all local minima are globally optimal~\cite{srebro2003weighted}, the behavior of first-order algorithms in solving this problem is not yet fully understood. In addition, establishing global convergence rates for gradient-based methods applied to a non-convex optimization problem is of interest on its own. 

Our control-theoretic approach provides guarantees and offers insight into an important class of non-convex optimization problems that are of interest in machine learning. Compared to existing literature~\cite{stoger2021small,soltanolkotabi2023implicit,stoger2024non} and more recent granular convergence guarantees~\cite{tarmoun2021understanding}, our approach reveals novel hidden structures and it serves as an important first step towards simpler convergence proofs for commonly-used variants of gradient descent. For example, stochastic gradient descent is still state-of-the-art in most vision applications and examining gradient descent is a well-established first step towards the analysis of these more complex algorithms.  We take a step toward providing insight into the global behavior of such algorithms by exploiting the structural properties of the underlying dynamics.

    In our analysis, we exploit the lifted coordinates that bring the system into a Riccati-like set of dynamical systems. Our main contribution lies in introducing a novel change of variables that brings the underlying dynamics into a cascade connection of three subsystems. We show that one of these subsystems is governed by autonomous dynamics which is decoupled from the rest of the system. In the over-parameterized regime, this subsystem is associated with the excess parameters; our analysis reveals its stability with an $O(1/t)$ asymptotic convergence rate, which captures the slow dynamics. For the other two subsystems, we utilize a Lyapunov-based approach to establish their exponential convergence under mild assumptions. 
    
    Our analysis offers new insight into the shape of the trajectories of first-order algorithms in the presence of over-parametrization. It also provides a complete characterization of their behavior by decoupling fast and slow dynamical components. In contrast to many existing works, we present our results in their most general form, without making any assumptions about the initialization or the structure of the target matrix. This generality enables our approach to be applied in a wider variety of settings, from small random initializations to large over-parameterized scenarios.
    
    The rest of this letter is as follows. In section~\ref{section-formulation}, we formulate \tcb{the} symmetric low-rank matrix factorization \tcb{problem} and introduce the corresponding gradient flow dynamics. In section~\ref{section-local-analysis}, we characterize the equilibrium points and examine their local stability properties. Additionally, we introduce the signal/noise decomposition of the optimization variable, which paves the way for global stability analysis. In Section~\ref{section-global-analysis}, \tcb{we introduce} a novel change of variables that transforms the original problem into a cascade connection of three subsystems, thereby facilitating the proof of global convergence. \tcb{We also} show that matrix recovery is impossible over certain invariant manifolds. Finally, we demonstrate that the optimal point is globally asymptotically stable if the initial condition lies outside the invariant manifold. In Section~\ref{section-remarks}, we offer concluding remarks, and in appendices, we provide proofs.

	\vspace*{-2ex}
	\section{Problem Formulation}\label{section-formulation}
 
	We study the matrix factorization problem,
	\be
		\label{eq.mainX}
		\minimize\limits_{X}  \:
		f(X) \; \DefinedAs \;
		(1/4) \, \norm{XX^T \, - \, M}_F^2
	\eeq
	where $M\in\R^{n\times n}$ is a given symmetric matrix, $X\in\R^{n\times r}$ is the optimization variable, and $\norm{\cdot}_F$ is the Frobenius norm. This is a non-convex optimization problem~\cite{stoger2021small}, and our objective is to find a matrix $X$ such that $XX^T$ approximates $M$. The gradient flow dynamics associated with~\eqref{eq.mainX} \mbox{are given~by}
	\begin{equation}
		\label{ode.gfd-X}
		\dot{X}
		\; = \; -\nabla_X f(X) \; = \;
		\left(M \; - \; XX^T\right)X.
	\end{equation}
	
    Let $M=V \Lambda V^T$ be the eigenvalue decomposition of the matrix $M$ where $V$ is the orthogonal matrix of eigenvectors and $\Lambda$ is the diagonal matrix of eigenvalues of $M$ with
    \beq
    \Lambda \; = \; \diag (\Lambda_1, - \Lambda_2).
    \label{eq.Lambda}
    \eeq
	Here, $\Lambda_1 \in \R^{r^\star\times r^\star}$ and $\tcb{-} \Lambda_2 \in \R^{(n-r^\star) \times (n-r^\star)}$ \tcb{are diagonal matrices of} positive 
	\tcb{
	$
		\{
		\lambda_1
		\ge 
		\lambda_2
		\ge 
		\ldots
		\ge 
		\lambda_{r^\star}
		> 
		0
		\}
		$}
and non-positive eigenvalues of $M$, \tcb{respectively}.
     
     \tcb{The positive definite part $V\diag(\Lambda_1,0)V^T$ of $M$ can be written as $X^\star X^{\star T}$ for some $X^{\star} \in\R^{n\times r}$ if and only if $r\ge r^\star$. We refer to the regimes with $r=r^\star$ and $r>r^\star$ as the exact and over-parameterized, respectively, indicating presence of excess parameters in the latter. To simplify the analysis, we define a new variable $Z \DefinedAs V^TX\in\R^{n\times r}$ and rewrite system~\eqref{ode.gfd-X} as}
    \be
		\label{ode.gfd-Z}
		\dot{Z}
		\; = \; 
		(\Lambda \; - \; ZZ^T)Z.
	\eeq
By analyzing local and global stability properties of the equilibrium points of system~\eqref{ode.gfd-Z}, we establish guarantees for the convergence of gradient flow dynamics~\eqref{ode.gfd-X} to the optimal value of non-convex optimization problem~\eqref{eq.mainX}. We also introduce a novel approach to characterize the convergence rate and offer insights into behavior of the dynamics in the over-parameterized regime and robustness of the recovery process.
	
	\vspace*{-2ex}
\section{Local stability analysis}\label{section-local-analysis}
	
	Let $\bar{Z}\bar{Z}^T$ with $\bar{Z}\in\R^{n\times \tcb{r}}$ be a low rank decomposition of the matrix \tcb{of positive eigenvalues of $M$}, i.e., 
	$
		\diag( \Lambda_1,0) = \bar{Z}\bar{Z}^T.
	$
    Then,  \tcb{globally optimal solutions to non-convex problem~\eqref{eq.mainX} are parameterized by $Z=\bar{Z}G$}, where $G\in\R^{\tcb{r\times r}}$ is an arbitrary unitary matrix. To resolve ambiguity caused by \tcb{this lack of uniqueness}, we introduce the lifted matrix $P=ZZ^T\in\R^{n\times n}$ and analyze properties of system~\eqref{ode.gfd-Z} without the need to explicitly consider $G$.
    
	   It is straightforward to verify that the matrix $P(t) \DefinedAs Z(t) Z^T(t)$ represents the unique solution of
	\be
	\label{ode.lift}
		\dot{P}
		\;=\;
		\left(\Lambda \; - \; P \right)P
		\;+\;
		P \left(\Lambda \; - \; P \right).
	\ee
	\tcb{In what follows, we utilize} the analysis of system~\eqref{ode.lift} \tcb{to deduce} convergence properties of system~\eqref{ode.gfd-Z}.
	
	\vspace*{-2ex}
	\subsection{Equilibrium points}
	

We \tcb{first identify} the equilibrium points of system~\eqref{ode.lift} and \tcb{characterize} their local stability properties. The set of equilibrium points of this system is determined by
	\be
	\label{eq.eqPointsDef}
		\cH
		\; \DefinedAs \,
		\left\{
		\diag (\bar{P}_1,0) \,\middle|\, \,\bar{P}_1 \, \succeq \, 0,~ \bar{P}_1\Lambda_1 \,=\, \bar{P}_1^2
		\right\}
	\ee
	where $\bar{P}_1\in\R^{r^\star\times r^\star}$; \tcb{see Appendix~\ref{app.equiPoints} for the proof. Two} trivial members of the equilibrium set $\cH$ are given by $\bar{P}_1=0$ and $\bar{P}_1=\Lambda_1$. Lemma~\ref{lem.eqPointsSecCharac} provides an alternative characterization of the set $\cH$ which we exploit in our subsequent analysis.
	
	
	\vsp
	
	\begin{mylem}
	\label{lem.eqPointsSecCharac}
		The positive semidefinite matrix $\bar{P}=\diag(\bar{P}_1,0)$ with $\bar{P}_1\notin\left\{0,\Lambda_1\right\}$ \tcb{is an equilibrium point of system~\eqref{ode.lift} if and only if,
		\be
		\bar{P}_1
		\;=\;
		U_{\mathbf{i}}D_{\mathbf{i}} U_{\mathbf{i}}^T.
		\label{eq.ViVo}
		\ee
Here,
			$
			\Lambda_1
			=
			\obt{U_{\mathbf{i}}}{U_{\mathbf{o}}}\diag(D_{\mathbf{i}},D_{\mathbf{o}}) \obt{U_{\mathbf{i}}}{U_{\mathbf{o}}}^T
		$
provides a unitary similarity transformation of $\Lambda_1$, $D_{\mathbf{i}}\in\R^{l\times l}$ and $D_{\mathbf{o}}\in\R^{(r^\star-l)\times(r^\star-l)}$ are diagonal matrices, $U_{\mathbf{i}}\in\R^{r^{\star} \times l}$ and $U_{\mathbf{o}}\in\R^{r^{\star}\times {(r^{\star}-l)}}$ with  $0<l<r^{\star}$ form a unitary matrix $U = \obt{U_{\mathbf{i}}}{U_{\mathbf{o}}}$, and 
	$ 
	\Lambda_1 - \bar{P}_1
	=
	U_{\mathbf{o}} D_{\mathbf{o}} U_{\mathbf{o}}^T.
	$} 
	\end{mylem}
	
	\vsp
	
	\begin{proof}
		See Appendix~\ref{app.equiPoints}.
	\end{proof}
	
	\vsp
	
	\tcb{The diagonal matrices $
	D_{\mathbf{i}}
	= 
	U_{\mathbf{i}}^T\Lambda_1 U_{\mathbf{i}}
	$
	and 
	$
	D_{\mathbf{o}}
	=
	U_{\mathbf{o}}^T \Lambda_1 U_{\mathbf{o}}
	$ 
partition the eigenvalues of $\Lambda_1$, i.e., 
	$
		\eig(\Lambda_1)
		=
		\eig(D_{\mathbf{i}}) \cup \eig(D_{\mathbf{o}}),
	$
into those that are in the spectrum of $\bar{P}$ and those that are not. If $\Lambda_1$ has distinct eigenvalues, then $U$ can only be a permutation matrix and system~\eqref{ode.lift} has exactly $2^{r^\star}$ equilibrium points. In this case, the corresponding matrices $\bar{P}_1$ are diagonal with entries determined by either $0$ or the eigenvalues $\lambda_k$ of $\Lambda_1$. If $M$ has a repeated positive eigenvalue, there is a lack of uniqueness of the eigenvectors of $\Lambda_1$, and thereby infinitely many matrices $\bar{P}_1$ of the form~\eqref{eq.ViVo}.}

   \vsp
   	
	\begin{example}
		\tcb{For $\Lambda_1 = \diag (3,2,1)$, apart from $\bar{P}_1 = \Lambda_1$ and $\bar{P}_1 = 0$, there are six additional members of the equilibrium set $\cH$, i.e.,
			$\bar{P}_1 \in \{ \diag(3,0,0$), $\diag(0,2,0)$, $\diag(0,0,1)$, $\diag(3,2,0)$, $\diag(3,0,1)$, $\diag(0,2,1) \}$.}
	\end{example}
	
	 \vsp
	\begin{remark}
	    \tcb{Lemma~\ref{lem.eqPointsSecCharac} implies that $\bar{P}_1 = \diag(\Lambda_1,0)$ is the only equilibrium point with rank $r^\star$. Since the remaining nonzero equilibrium points are given by $U_{\mathbf{i}}D_{\mathbf{i}}U_{\mathbf{i}}^T$ for some diagonal $D_{\mathbf{i}}\in \R^{l\times l}$ with $l < r^{\star}$, their rank is smaller than $r^\star$.}
	\end{remark}

	\vspace*{-2ex}
	\subsection{Local stability properties}
	
	To examine local stability properties, we linearize system~\eqref{ode.lift} around its equilibrium point $\bar{P}\in\cH$. \tcb{Substituting} $P=\bar{P}+\epsilon\tilde{P}$ to~\eqref{ode.lift} and keeping $O (\epsilon)$ terms \tcb{leads to the linearized dynamics},
	\be
	\label{ode.lift_linearized}
		\dot{\tilde{P}}
		\;=\;
		(\Lambda \; - \; 2\bar{P} )\tilde{P}
		\;+\;
		\tilde{P} (\Lambda \; - \; 2\bar{P} )
	\ee
\tcb{which allows us} to characterize stability properties for $\bar{P}_1\neq \Lambda_1$.	
	
	\vsp
	
	\begin{myprop}
		Any equilibrium point $\bar{P}=\diag(\bar{P}_1,0)\in\cH$ of system~\eqref{ode.lift} with $\bar{P}_1\neq\Lambda_1$ is unstable.
	\end{myprop}
	
	
	\begin{proof}
		Let  $\bar{P}_1$ \tcb{be parameterized by} $(D_{\mathbf{i}},D_{\mathbf{o}},U_{\mathbf{i}},U_{\mathbf{o}})$, \tcb{as described} in Lemma~\ref{lem.eqPointsSecCharac}. 
		Now, consider the mode $q\DefinedAs u^T\tilde{P} u$, where $u \DefinedAs \, [u_k^T ~0 \, ]^T \in \R^{n}$ and  $u_k$ is the $k$th column of $U_{\mathbf{o}}$. For \tcb{linearized} system~\eqref{ode.lift_linearized}, we have
		\be
			\dot{q}
			\, = \,
			2u^T (-\bar{P}\,+\,(\Lambda\,-\,\bar{P}) ) \tilde{P}u
			\, = \,
			2 u^T (\Lambda\,-\,\bar{P})\tilde{P} u
			\, = \,
			2 \lambda q
			\non
		\ee
		where $\lambda>0$ is an eigenvalue of $\Lambda_1$ that corresponds to the $k$th diagonal entry of $D_{\mathbf{o}}$. \tcb{Since $U_{\mathbf{i}}$ and $U_{\mathbf{o}}$ are mutually orthogonal,} the second equality follows from $u^T\bar{P}=u_k^T U_{\mathbf{i}}D_{\mathbf{i}}U_{\mathbf{i}}^T=0$, \tcb{and the} last equality follows from the second equation in~\eqref{eq.ViVo}. Thus, $q$ is an unstable mode of the linearized system, and $\bar{P}$ is an unstable equilibrium point of nonlinear system~\eqref{ode.lift}. The analysis for $\bar{P}=0$ is similar and is omitted for brevity.
	\end{proof}
	
	\vsp

Lemma~\ref{lem.GapToOtherEQP} proves that $\diag(\Lambda_1,0)$ is an isolated equilibrium point by establishing a lower bound on the distance between $\diag(\Lambda_1,0)$ and all other equilibria of system~\eqref{ode.lift}. 

		
	
	\vsp
	
	\begin{mylem}\label{lem.GapToOtherEQP}
		For any equilibrium point $\bar{P} \neq \diag(\Lambda_1,0)$ of system~\eqref{ode.lift}, we have
		$
			\norm{\bar{P} - \diag(\Lambda_1,0)}_2
			\ge
			\lambda_{r^\star},
		$
		where $\lambda_r^\star>0$ is the smallest  eigenvalue of $\Lambda_1$ and $\norm{\cdot}_2$ is the spectral norm.
	\end{mylem}
	
	\vsp
	
	\begin{proof}
		The result is trivial for $\bar{P}=0$. For $\bar{P}\neq 0$,  \tcb{Lemma~\ref{lem.eqPointsSecCharac} implies that}
		$
		\norm{\bar{P} - \diag (\Lambda_1,0)}_2
			=
			\norm{\bar{P}_1\,-\,\Lambda_1}_2
			=
			\norm{D_{\mathbf{o}}}_2
			\ge
			\lambda_{r^\star},
		$ 
		where the inequality follows \tcb{because} the diagonal matrix $D_{\mathbf{o}}$ contains a nonempty subset of the eigenvalues of $\Lambda_1$.
	\end{proof}
	
	\vsp
		
	\tcb{Using} Lemma~\ref{lem.GapToOtherEQP} in conjunction with a Lyapunov-based argument, \tcb{we prove local asymptotic stability} of the global minimum $\bar{P}= \diag(\Lambda_1,0)$ of optimization problem~\eqref{eq.mainX}. 

	\vsp
	
	\begin{myprop}
	\label{prop.localStability}
		The isolated equilibrium point 
		$
		\bar{P}= \diag(\Lambda_1,0)
		$
		of system~\eqref{ode.lift} is locally asymptotically stable.
	\end{myprop}
	
	\vsp
	
	\begin{proof}
		\tcb{See Appendix~\ref{app.localstabResults}.}
	\end{proof}
	
	\vsp
	
 \tcb{Proposition~\ref{prop.localStability} proves local asymptotic stability of 
 	$
		\diag(\Lambda_1,0)
		$
using the objective function in~\eqref{eq.mainX} as a Lyapunov function candidate, i.e., $V_F(P)=f(X)$. However, this approach does not provide insight into global stability properties because of the presence of additional equilibrium points. } 
 	
	\vspace*{-2ex}
	\subsection{Signal/noise decomposition}
	
To prove global asymptotic stability \tcb{of 
    $
    \bar{P}= \diag(\Lambda_1,0),
    $} 
we decompose $Z$ into \tcb{``signal'' $Z_1$ and ``noise'' $Z_2$} components, 
    \be
    Z^T 
    \, = \,
    \obt{Z_1^T}{Z_2^T}
    \label{eq.Zsn}
    \ee
\tcb{where} $Z_1\in\R^{r^\star\times r}$ captures the part of \tcb{$Z$ that is desired to get aligned with the positive definite low-rank structure, i.e, $\lim_{t\rightarrow\infty}Z_1(t) Z_1^T(t)= \Lambda_1$,} and $Z_2\in\R^{(n-r^\star)\times r}$ is \tcb{supposed to vanish asymptotically}. This decomposition allows us to isolate contribution of ``noise'', thereby facilitating analysis of the stability properties of the gradient flow dynamics.

Using decomposition~\eqref{eq.Zsn} of $Z$ we can rewrite system~\eqref{ode.gfd-Z} as
	\begin{subequations}\label{ode.Z1-Z2}
		\begin{align}\label{ode.Z1}
			\dot{Z}_1
			&\;=\;
			(\Lambda_1 \; - \; Z_1Z_1^T)Z_1 \; - \; Z_1 Z_2^TZ_2
			\\[0.cm]
			\label{ode.Z2}
			\dot{Z}_2
			&\;=\;
			-(\Lambda_2 \; + \; Z_2Z_2^T)Z_2 \; - \; Z_2 Z_1^TZ_1
		\end{align}
	\end{subequations}
\tcb{and by partitioning $P$ conformably with the partition of $Z$,
	\[
	P
	\;=\;
	\tbt{P_1}{P_0}{P_0^T}{P_2}
	\;\DefinedAs\;
	\tbt{Z_1Z_1^T}{Z_1Z_2^T}{Z_2Z_1^T}{Z_2Z_2^T}
	\;=\;
	Z Z^T
	\]
the corresponding lifted system~\eqref{ode.lift} can be written as,
	\begin{subequations}\label{ode.P1-P0-P2}
		\begin{align}\label{ode.P1-decomposed}
			\dot{P}_1
			&\;=\;
			P_1\Lambda_1
			\,+\,
			\Lambda_1P_1 
			\,-\,
			2(P_1^2 \,+\, P_0 P_0^T)
			\\[0.cm]\label{ode.P0-decomposed}
			\dot{P}_0
			&\;=\;
			\Lambda_1P_0
			\,-\,
			P_0\Lambda_2
			\,-\,2(P_1P_0\,+\,P_0P_2)
			\\[0.cm]\label{ode.P2-decomposed}
			\dot{P}_2
			&\;=\;
			-P_2\Lambda_2
			\,-\,
			\Lambda_2P_2
			-2(P_0^TP_0 + P_2^2).
		\end{align}
	\end{subequations}
Here, $P_1$ and $P_2$ represent the ``signal'' and ``noise'' components of $P$, and $P_0$ accounts for the coupling between $Z_1$ and $Z_2$.}	
We note that for $r^\star=n$ the matrices $Z_2$, $P_2$, and $P_0$ disappear.	
	
	
	\vspace*{-2ex}
	\subsection{\tcb{The evolution of the ``noise'' component $P_2$}}


We next \tcb{examine the evolution} of the noise component \tcb{$P_2 \DefinedAs Z_2Z_2^T$} and \tcb{identify the conditions under which system~\eqref{ode.lift} converges to the stable equilibrium point $\bar{P}=\diag(\Lambda_1,0)$. We show} that $P_2 (t)$ decays with time and that it asymptotically vanishes as the system approaches the stable equilibrium point.
	
	For any initial condition, we first prove that the spectral norm
	of the noise component $P_2$,
	\begin{align}\label{eq.noiseMagnitude}
		V_N(P_2)
		\;\DefinedAs\;
		\norm{P_2}_2
		\;=\;
		\norm{Z_2}_2^2
	\end{align}
converges to zero with a polynomial rate.
	
	\vsp
	
	\begin{mylem}\label{lem.sym-noise-decrease}
		Along the trajectories of system~\eqref{ode.lift}, the \tcb{spectral norm of the noise component $P_2 (t)$} satisfies
		\[
		V_N(P_2(t))
		\;\le\;
		2 V_N(P_2(\tau)) / (1 \, + \, t \, - \, \tau),
		~~
		\tcb{\forall \, t \, \ge \, \tau \, > \, 0.}
		\]
	\end{mylem}

	
		\begin{proof}
		\tcb{See Appendix~\ref{app.localstabResults}.}
	\end{proof}

	\vsp 
	
	Lemma~\ref{lem.sym-noise-decrease} implies that $V_N(P_2(t))$ is a decreasing function of time that converges to zero with $O(1/t)$ rate. \tcb{While we demonstrate in Section~\ref{section-global-analysis} that this result can be improved to exponential convergence for $r=r^\star$, we  next establish its sharpness for the over-parameterized regime. }
	
	\vsp
	
	\subsubsection*{\tcb{Sharpness of Lemma~\ref{lem.sym-noise-decrease} in the over-parameterized regime}}
\tcb{If $r>r^\star$, we demonstrate through an example that the rate $O(1/t)$ in Lemma~\ref{lem.sym-noise-decrease} cannot be improved.
	In particular, for $\Lambda_2=0$, } we identify an invariant manifold of system~\eqref{ode.lift} over which \tcb{the magnitude of the noise component $P_2 (t)$} decreases with $O (1/t)$ rate. In particular, let $Z^T (0) = [ \, Z_1^T (0) ~ Z_2^T(0) \, ]$ be such that $P_1(0)=\Lambda_1$, $P_0(0)=0$, and $P_2(0)\neq 0$; \tcb{e.g.,} this can be achieved with 
	$
	Z_1(0)
	=
	[ \, \Lambda_1^{1/2} ~~ 0_{r^\star\times(r-r^\star)} \, ]
	$
where the rows of  $Z_2(0)\neq 0$ belong to the orthogonal complement of the row space of $Z_1(0)$, \tcb{which is nonempty because $r> r^{\star}$}. It is now easy to verify that $\dot{P}_1=0$ and $\dot{P}_0=0$, \tcb{thereby implying}
	$P_1(t)=\Lambda_1$ and $P_0(t)=0$ for all $t \geq 0$. \tcb{In this case,} system~\eqref{ode.lift} simplifies to
	$
	\dot{P_2}
	=
	-
	2 P_2^2,
	$
the eigenvectors of $P_2$ remain unchanged, and each nonzero eigenvalue $\mu>0$ of $P_2$ satisfies $\mu(t)= 2\mu(0)/(1+t)$.
	
	\vspace*{-1ex}
	\section{Main result: global stability analysis}
	\label{section-global-analysis}
	
\tcb{Our main theoretical result establishes global convergence of the gradient flow dynamics to $\diag(\Lambda_1,0)$ if and only if the initial condition $P(0)$ satisfies
    $
	\rank(P_1(0))
	=
	r^\star.
	$
This condition guarantees recovery of the optimal low-rank positive semi-definite factorization of the target matrix $M = M^T$.}

 	\vsp
	
	\begin{mythm}\label{thm.conv-original-sym}
		\tcb{For any initial condition $P(0)$ of system~\eqref{ode.lift} with $P_1(0)\succ 0$, let the time $t$ be large enough for the exponentially decaying term in function
		\[
		l(t)
		\DefinedAs
		\left(
		\norm{P_1^{-1}(0) - \Lambda_1^{-1}}_2
		\,+\,
		2\,t\, \norm{P_1^{-1}(0)P_0(0)}_2^2
		\right)
		\mre^{-2\lambda_{r^\star} t}
		\]		
to dominate the linearly growing term such that $l(t)\le 1/(2\lambda_1)$. Then, $P_1(t)\succ 0$ for all times and the solution to~\eqref{ode.lift} satisfies
		\begin{align*}
			\norm{P_1(t) - \Lambda_1}_2
			&\;\le\;
			2\,\lambda_1^2 \, l(t)
			\\[0.cm]
			\norm{P_0(t)}_2
			&\;\le\;
			\norm{P_1^{-1}(0)\,P_0(0)}_2
			\left(\lambda_1 + 2\,\lambda_1^2 \, l(t)\right)
			\mre^{-\lambda_{r^\star} t}.
		\end{align*} }
	\end{mythm}
	
	\begin{proof}
		\tcb{See Appendix~\ref{app.globalstabProp}.}
	\end{proof}
			
	\vsp

        \tcb{Theorem~\ref{thm.conv-original-sym} establishes exponential asymptotic convergence of $P_1(t)$ to $\Lambda_1$ and of $P_0(t)$ to zero at a rate determined by the smallest positive eigenvalue $\lambda_{r^\star}$ of $\Lambda_1$, provided that the initial condition satisfies $P_1(0)\succ 0$. In conjunction with Lemma~\ref{lem.sym-noise-decrease}, this proves global convergence to the optimal solution. }
        	
	\vsp
        	
   
  {\color{black}  \begin{myrem}
    [Exponential convergence for the $r=r^\star$]\label{exact.param.conv}
   
   For the exact parameterization with $r=r^\star$, the matrix $Z_1$ remains invertible under the conditions of Theorem~\ref{thm.conv-original-sym} and it satisfies $P_1^{-1} = Z_1^{-T}Z_1^{-1}$. This allows us to write
   	\[
   	P_2 \;=\; Z_2Z_2^T \;=\;   Z_2Z_1^T P_1^{-1} Z_1Z_2^T \;=\; P_0^T P_1^{-1}P_0^T.
   	\]
   	Combining this equality with Theorem~\ref{thm.conv-original-sym}, it is straightforward to verify that $P_2(t)$ vanishes at the exponential rate $-2\lambda_r$.
\end{myrem}}
	\vsp
     {\color{black}\begin{mycor}\label{conv.Zcoor}
      For any full row rank matrix $Z_1(0)\in\R^{r^{\star}\times r}$ and arbitrary $Z_2(0)\in\R^{(n-r^{\star})\times r}$,  there exists  a point $Z^{\star} = [Z^{\star T}_1~0]^T\in\mathcal{Z}^{\star}$ and scalars $c_i>0$
     such that system~\eqref{ode.gfd-Z} 
      with the initial condition $Z(0)=\obt{Z_1^T(0)}{Z_2^T(0)}^T$ satisfies
     \[
     \norm{Z_1(t)-Z_1^\star}_2
     \; \le \; 
     c_1\mre^{-\lambda_{r^\star}t},
     ~~
     \norm{Z_2(t)}
     \; \le \; c_2/\sqrt{t}
     \]
     where
      $
      \mathcal{Z}^{\star}
      \DefinedAs
      \{ Z^{\star}\in\R^{n\times r} \, | \, Z^{\star}Z^{\star T}=\diag(\Lambda_1,0) \}
      $
      is the set of optimal solutions of problem~\eqref{eq.mainX} in the $Z=V^TX$ coordinates. 
      Furthermore, for exact parameterization $r=r^\star$, the bound on $Z_2(t)$ improves to $\norm{Z_2(t)}_2\le c_2\mre^{-\lambda_{r^\star}t}$.
    \end{mycor}  }
   	
	\vsp
	\begin{proof}
	See Appendix~\ref{app.globalstabProp}.
	\end{proof}
   
	\vsp
	
\begin{myrem}
	\tcb{The condition $P_1(0)\succ 0$ in Theorem~\ref{thm.conv-original-sym} holds if and only if the matrix $X^{\star T}X(0)$ has rank $r^\star$, where $X^\star\in\R^{n\times r}$ is an optimal solution to problem~\eqref{eq.mainX} of rank $r^\star$, and $X(0)\in\R^{n\times r}$ is the initial condition of~\eqref{ode.gfd-X}. This can be interpreted as the matrices $X(0)$  and $X^\star$ being {\it well-aligned}.}
\end{myrem}	

    \vsp
    
	\tcb{	
The key technical result that allows us to prove Theorem~\ref{thm.conv-original-sym} is a novel nonlinear change of variables that we introduce in Section~\ref{sec.H}. In Section~\ref{sec.invariant}, we show that for any rank deficient initial condition $P_1(0)$, the trajectory $P(t)$ of system~\eqref{ode.lift} belongs to an invariant subspace that does not contain stable equilibrium point $\diag(\Lambda_1,0)$ and, in Section~\ref{sec.Hstability}, we provide convergence guarantees in the new set of coordinates.
	}
	

	\vspace*{-2ex}
	\subsection{Change of variables}
	\label{sec.H}
	
	We next introduce a nonlinear change of variables that simplifies the analysis of system~\eqref{ode.lift} and facilitates the proof of global convergence to \tcb{$\bar{P} = \diag(\Lambda_1,0)$ for a full-rank $P_1(0)$.}
		
		\vsp
		
	\begin{myprop}\label{prop.changeVariables}
		Let $P_1(t)$ be invertible for some $t>0$.
		The evolution of \tcb{the} matrices,
		\begin{equation}\label{eq.changevar-nonlin}
			\ba{rclcl}
			H_1
			&\!\!\!\DefinedAs\!\!\!&
			P_1^{-1}
			&\!\!\!\in\!\!\!&
			\R^{r^\star\times r^\star}
			\\[0.cm]
			H_0
			&\!\!\! \DefinedAs \!\!\!&
			P_1^{-1}P_0
			&\!\!\!\in\!\!\!&
			\R^{r^\star\times(n-r^\star)}
			\\[0.cm]
			H_2
			&\!\!\! \DefinedAs \!\!\!&
			P_2\,-\,P_0^TP_1^{-1} P_0
			&\!\!\!\in\!\!\!&
			\R^{(n-r^\star)\times(n-r^\star)}
			\ea
		\end{equation}
		is governed by the following dynamical system
		\begin{subequations}\label{ode.cascade-full}
			\begin{align}\label{ode.H1-dynamics}
				\dot{H}_1
				&\;=\;
				-\Lambda_1 H_1 \,-\, H_1\Lambda_1
				\,+\, 2 ( I \,+\, H_0 H_0^T )
				\\[0.cm]
				\label{ode.H0-dynamics}
				\dot{H}_0
				&\;=\;
				- \Lambda_1H_0 
				\,-\,
				H_0
				(
				\Lambda_2
				\, + \,
				2 H_2
				)
				\\[0.cm]
				\label{ode.H2-dynamics}
				\dot{H}_2
				&\;=\;
				-\Lambda_2 H_2
				\,-\,
				H_2\Lambda_2 
				\,-\,
				2 H_2^2.
			\end{align}
		\end{subequations} 
	\end{myprop}

	\begin{proof}
		See Appendix~\ref{app.globalstabProp}.
	\end{proof}

	\vsp
	
\begin{figure}[h]
	\begin{center}
		\begin{tabular}{c}
			\resizebox{8.5cm}{!}{
				\begin{tikzpicture}
					\node[] (pic) at (0,0) {\includegraphics[]{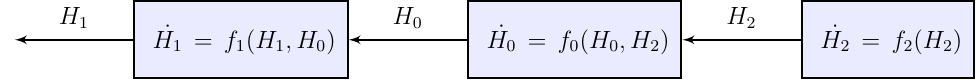}};
				\end{tikzpicture}
			}
		\end{tabular}
	\end{center}
	\caption{\tcb{Block diagram of system~\eqref{ode.cascade-full} illustrating its cascade structure.}}
	\label{fig.Hbd}
\end{figure}

In the $H$-coordinates, we have a cascade connection of three subsystems; \tcb{see Fig.~\ref{fig.Hbd} for an illustration. While} the $H_2$-dynamics are not influenced by the evolution of $H_1 (t)$ and $H_0 (t)$, $H_2 (t)$ enters as a coefficient into the $H_0$-dynamics and $H_0 (t) H_0^T (t)$ enters as an additive input into the $H_1$-dynamics. It is also worth noting that the matrix $H_0 (t)$ which satisfies Sylvester equation~\eqref{ode.H0-dynamics} evolves independently of $H_1(0)$. 

	\vsp
	\begin{myrem} 
	\tcb{The matrix $H_1$ is determined by the inverse of $P_1$,
	$H_2 = P/P_1$ is the Schur complement of the matrix $P$ with respect to $P_1$, and $H_0$ in~\eqref{eq.changevar-nonlin} is the best matrix that transforms the signal component $Z_1$ to the noise component $Z_2$, i.e.,}
		\begin{align*}
			H_0\;=\;\argmin_{E \, \in \, \R^{r^\star \times (n-r^\star)}} \norm{Z_2\,-\,E^T Z_1}_F^2.
		\end{align*}
  	\end{myrem}
    
    	\vspace*{-2ex}
	\subsection{Invariant subspaces} 
	\label{sec.invariant}

        We next show that the null space of $P_1 (t)$ does not change with time if $P_1(0)$ is singular. In this case, system~\eqref{ode.lift} is restricted to a subspace and convergence to the full-rank solution $\bar{P}_1=\Lambda_1$ is not possible. In other words, initial conditions with rank-deficient $P_1(0)$ lead to incomplete recovery. To show this, we note that for any fixed $v$ with $v^TP_1v=0$ we can write
     \[
		v^T\dot{P}_1 v 
		\;=\;
		2v^T (P_1\Lambda_1
		\,-\,
		(P_1^2 \,+\, P_0 P_0^T) ) \, v
		\;=\;
		0.
	\]
This expression follows from the fact that such $v$ also belongs to the null space of $Z_1^T$ and, hence, $P_0^Tv=0$. This implies that $v^TP_1(t)v$ remains equal to zero for all $t\geq 0$.
	
	\vspace*{-2ex}
\subsection{Convergence guarantees in the $H$-coordinates}
	\label{sec.Hstability}

\tcb{Cascade connection~\eqref{ode.cascade-full} in transformed coordinates~\eqref{eq.changevar-nonlin} allows us} to examine stability properties of each $H$-component. For $H_2$-subsystem~\eqref{ode.H2-dynamics}, we show global asymptotic stability of the origin with a worst-case convergence rate $O(1/t)$, which is achieved when $\Lambda_2=0$. We also establish that $H_1 (t)$ and $H_0 (t)$ converge exponentially to $\Lambda_1^{-1}$ and $0$, respectively. Theorem~\ref{thm.analytical-sol-symmetric-H1} provides analytical expressions for $H_1 (t)$ and $H_0 (t)$ that satisfy~\eqref{ode.cascade-full} in terms of $H_0 (t)$ and $H_2 (t)$, respectively, and it proves the aforementioned convergence rates.

	\vsp
	
	\begin{mythm}\label{thm.analytical-sol-symmetric-H1}
		For any full rank initial condition $H_1(0)\succ 0$, the solution $H_1(t)$ to~\eqref{ode.cascade-full} is given by
		\begin{align*}
			H_1(t)
			&\;=\;
			\Lambda_1^{-1}
			\,+\,
			\mre ^{-\Lambda_1t}
			\left(
			H_1(0)
			\,-\,
			\Lambda_1^{-1}
			\right)
			\mre ^{-\Lambda_1t}
			\\[0.cm]
			&\;~\;
			+\;2\int_{0}^{t}
			\mre ^{\Lambda_1(\tau-t)}
			H_0(\tau)H_0^T (\tau)
			\mre ^{\Lambda_1(\tau-t)}
			\,\mrd \tau
		\end{align*}
		and it satisfies $H_1(t)\succ 0$. 
Moreover, for any $t\ge 0$,  we have
		\begin{align*}
			\norm{\tilde{H}_1(t)}_2
			&\;\le\;
			(
			\norm{\tilde{H}_1(0)}_2
			\,+\,
			2\,t
			\norm{H_0(0)}_2^2
			)
			\,
			\mre^{-2\lambda_{r^\star} t}
			\\[0.cm]
			\norm{H_0(t)}_2
			&\;\le\;
			\norm{H_0(0)}_2\,
			\mre^{-\lambda_{r^\star} t}
			\\[0.cm]
			\norm{H_2(t)}_2
			&\;\le\;
			2\norm{H_2(0)}_2/(1 \, + \, t)
		\end{align*}
		where $\tilde{H}_1(t)\DefinedAs H_1(t)-\Lambda_1^{-1}$. 	
	\end{mythm}
	
	
	\begin{proof}
		See Appendix~\ref{app.globalstabProp}.
	\end{proof}


	Theorem~\ref{thm.analytical-sol-symmetric-H1} demonstrates that starting from any initial condition $H(0)$ with $H_1(0)\succ0$,  
$H_1(t)$ and $H_0(t)$ converge exponentially to $\Lambda_1^{-1}$ and $0$ with the respective rates $2\lambda_{r^\star}$ and $\lambda_{r^\star}$. Note that the upper bound on $\norm{H_1(t)-\Lambda_1^{-1}}_2$ involves algebraic growth for small $t$, i.e., $\log(t)/t\le\lambda_r$, but the exponential decay eventually dominates. In Remark~\ref{rem.algebraic growth}, we provide examples to demonstrate that this algebraic growth in the upper bound on $\norm{H_1-\Lambda_1^{-1}}_2$ cannot be eliminated.
	
	\vsp
	
	\subsubsection*{\tcb{Exact} parameterization}
	We now specialize Theorem~\ref{thm.analytical-sol-symmetric-H1} to $r=r^\star$.
	In this case, the matrix $P_1=Z_1Z_1^T$ is invertible if and only if $\det (Z_1) \neq 0$. Thus, change of variables~\eqref{eq.changevar-nonlin} satisfies
	\[
		H_1
		\, = \,  
		P_1^{-1}
		\, = \, 
		Z_1^{-T}Z_1^{-1},
		~~
		H_0
		\, = \,  
		P_1^{-1}P_0
		\, = \, 
		Z_1^{-T}Z_2^T
		\]
and  the Schur complement $H_2 = P/P_1$ vanishes,
	\[
		H_2 
		= 
		P_2 - P_0^TP_1^{-1}P_0
		= 
		Z_2Z_2^T - Z_2Z_1^T (Z_1^{-T}Z_1^{-1} )Z_1Z_2^T
		= 
		0.
	\]

	\begin{mycor}\label{prop.analytical-solution-r-equal-rstar}
		\tcb{For $r=r^\star$, the solution of~\eqref{ode.cascade-full} is} given by
		\begin{align*}
			\tilde{H}_1(t)
			&\;=\;
			\mre ^{-\Lambda_1t}
			(
			\tilde{H}_1(0)
			\,+\,
			2H_0(0)
			\Psi(t)
			H_0^T (0)
			)
			\mre ^{-\Lambda_1t}
			\\[0.cm]
			H_0(t)
			&\;=\;
			\mre^{-\Lambda_1t} H_0(0)\mre^{-\Lambda_2 t},
			~~ 
			H_2 (t)
			\;=\;
			0
		\end{align*}
		where $\tilde{H}_1(t)\DefinedAs H_1(t)-\Lambda_1^{-1}$ and $\Psi(t)\DefinedAs\int_{0}^{t}\mre ^{-2\Lambda_2 \tau}
		\,\mrd \tau$.
	\end{mycor}
	
	\vsp
	
	\begin{proof}
		Since $H_2 = P/P_1=0$ for $r=r^\star$, the $H_0$-dynamics in~\eqref{ode.cascade-full} simplify to
		$
			\dot{H}_0
			=
			-\,\Lambda_1
			H_0
			- 
			H_0 \Lambda_2
		$
		which leads to the analytical solution for $H_0(t)$ in \tcb{Corollary}~\ref{prop.analytical-solution-r-equal-rstar}.
		Combining this result with Theorem~\ref{thm.analytical-sol-symmetric-H1} yields the analytical solution for $H_1(t)$ and completes the proof.
	\end{proof}
	
	\vsp
	
	\begin{myrem}\label{rem.algebraic growth}
		For $\Lambda_2=0$, the matrix $\Psi(t)$ in Proposition~\ref{prop.analytical-solution-r-equal-rstar} becomes equal to $tI$, thereby showing that the algebraic growth in the upper bound established in Theorem~\ref{thm.analytical-sol-symmetric-H1}  is inevitable. 
	\end{myrem}
	
	\vsp
	
		\begin{myrem}
			\label{rem.cascade-expanded}
			While changes of coordinates~\eqref{prop.changeVariables} and~\eqref{eq.changevar-nonlin} are obtained by decomposing the optimization variable $Z$ into $Z_1\in \R^{i\times r}$ and $Z_2\in\R^{(n-i)\times r}$ with $i=r^\star$, and  $\Lambda=\diag(\Lambda_1,-\Lambda_2)$ into its positive and non-positive diagonal blocks, it is easy to verify that  a similar change of variables can be introduced for any $i\in\{1,\ldots,n-1\}.$ In addition,  the $H_2$-dynamics in~\eqref{ode.H2-dynamics} have a similar structure to the original system~\eqref{ode.lift} that governs the $P$-dynamics. Now, let $\Lambda_1=\diag(\hat{\lambda}_1 I,\ldots,\hat{\lambda}_k I)$ be the partitioning of $\Lambda_1$ based on its $k$ distinct eigenvalues $\hat{\lambda}_1>\cdots>\hat{\lambda}_k>0$. Utilizing the above facts, we can start from the principle eigenspace of $\Lambda$ associated with $\hat{\lambda}_1$ and successively employ $k$ changes of variables similar to~\eqref{eq.changevar-nonlin} to decouple individual eigenspaces and bring system~\eqref{ode.lift} into a cascade connection of $k+2$ subsystems,
			\begin{subequations}\label{ode.cascade-full--expanded}
				\begin{align}\label{ode.H1-dynamics--expanded}
					\dot{\hat{H}}_{i,1}
					&\;=\;
					-2\hat{\lambda}_i \hat{H}_{i,1}
					\,+\, 2 ( I \,+\, \hat{H}_{i,0} \hat{H}_{i,0}^T )
					\\[0.cm]
					\label{ode.H0-dynamics--expanded}
					\dot{\hat{H}}_{i,0}
					&\;=\;
					- \hat{\lambda}_i\hat{H}_{i,0} 
					\,+\,
					\hat{H}_{i,0}
					\hat{\Lambda}_i
					\,-\,
					2\hat{H}_{i,0}
					\hat{H}_{i+1,1}
					\\[0.cm]
					\label{ode.H2-dynamics-expanded}
					\dot{\hat{H}}_2
					&\;=\;
					-\Lambda_2 \hat{H}_2
					\,-\,
					\hat{H}_2\Lambda_2 
					\,-\,
					2 \hat{H}_2^2.
				\end{align}
			\end{subequations} 
for $i=1,\ldots,k$.
Here, 
the pair $(\hat{H}_{i,1}, \hat{H}_{i,0})$  corresponds to the distinct eigenvalue $\hat{\lambda}_i>0$,   $\hat{H}_{k+1,i}\DefinedAs\hat{H}_2$ corresponds to non-positive eigenvalues of $\Lambda$, and $\hat{\Lambda}_i\DefinedAs\diag(\hat{\lambda}_i I,\ldots,\hat{\lambda}_k I,-\Lambda_2)$ is a lower diagonal block of $\Lambda$; see Fig.~\ref{fig.cascade-principle-eigenspaces} for an illustration.

Furthermore, as we demonstrate in Appendix~\ref{app.expanded-decomp}, the autonomous system in~\eqref{ode.H2-dynamics-expanded} satisfies $\hat{H}_2(t)=H_2(t)$, where $H_2$ is the original Schur complement in Theorem~\ref{thm.analytical-sol-symmetric-H1}. Finally, using similar arguments as in the proof of Theorem~\ref{thm.analytical-sol-symmetric-H1}, it is straightforward to show under the conditions of Theorem~\ref{thm.analytical-sol-symmetric-H1}, 
\begin{align*}
	\norm{\hat{H}_{i,1}(t)-\hat{\lambda}^{-1}_iI}_2
	&\;\le\;
	\hat{c}_{i,1}\,
	\mre^{-2\lambda_{i} t}
	\\[0.cm]
	\norm{\hat{H}_{i,0}(t)}_2
	&\;\le\;
	\hat{c}_{i,0}\,
	\mre^{-(\hat{\lambda}_i-\hat{\lambda}_{i+1}) t}
	\\[0.cm]
	\norm{\hat{H}_2(t)}_2
	&\;\le\;
	\hat{c}_{2}/(1 \, + \, t)
\end{align*}
where $\hat{c}_{i,1}$, $\hat{c}_{i,0}$, and $\hat{c}_{2}$ are positive scalars that depend on the initial condition, and $\hat{\lambda}_{k+1}\DefinedAs0$. This decomposition demonstrates the impact of gaps between eigenvalues of $\Lambda_1$ on the convergence behavior of system~\eqref{ode.lift}. 
		\end{myrem}
		\begin{figure}[h]
			\begin{center}
				\begin{tabular}{c}
					\resizebox{8.5cm}{!}{
						\begin{tikzpicture}
							\node[] (pic) at (0,0) {\includegraphics[]{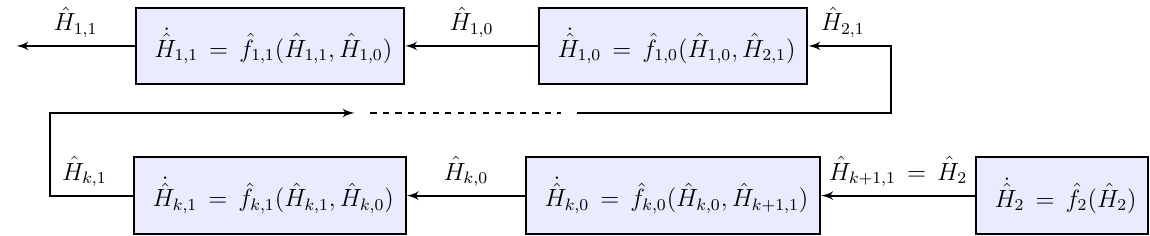}};
						\end{tikzpicture}
					}
				\end{tabular}
			\end{center}
			\caption{\tcb{Block diagram of the $2k+1$ subsystems in~\eqref{ode.cascade-full--expanded}.}}
			\label{fig.cascade-principle-eigenspaces}
		\end{figure}	

\vsp 

	\begin{myrem}
		When $r=r^\star$ and $\Lambda_2=0$, 
		we can introduce a new variable	$W=H_0H_0^T$
		to reduce~\eqref{ode.cascade-full} to \tcb{a stable LTI system with stability margin $\lambda_{r^\star}$,}
		\[
		\dot{W}
		=
		- \Lambda_1W -  W\Lambda_1,
		~
		\dot{H}_1
		=
		-\Lambda_1 H_1 - H_1\Lambda_1 + 2 I + 2 W.
		\]
	\end{myrem}
	
	\vspace*{-2ex}
	\section{\tcb{Computational experiments}}
	\label{section-computation}	
	
    Herein, we provide an example to demonstrate the merits of our theoretical results. We set $n=10$, $r^\star=4$, $r=8$, $\Lambda_1=\diag(4,3,2,1)$, and $\Lambda_2 =0$. The black, blue, green, red, and yellow curves mark trajectories of gradient flow dynamics~\eqref{ode.gfd-X} with random initial conditions around the equilibrium points $\bar{P}_1=\diag(0,0,0,0)$, $\bar{P}_1=\diag(4,0,0,0)$, $\bar{P}_1=\diag(4,3,0,0)$, and $\bar{P}_1=\diag(4,3,2,0)$, and a normal random initialization, respectively. Figure~\ref{fig.1} illustrates the objective function in~\eqref{eq.mainX}. 
    Figures~\ref{fig.2} and~\ref{fig.3} demonstrate the exponential convergence of $P_1 (t)$ to $\Lambda_1$ and of $P_0 (t)$ to $0$, respectively. While at early stages, we observe a transient behavior, for large enough $t$, the error converges exponentially. 
   Finally, Figure~\ref{fig.4} shows the sub-linear convergence resulting from over-parameterization. We observe that the convergence curves confirm the results established in Theorem~\ref{thm.conv-original-sym}.
 \begin{figure*}[t!]
  \centering
  \begin{tabular}{c@{\hspace{-0.4 cm}}c@{\hspace{-0.1 cm}}c@{\hspace{-0.4 cm}}c@{\hspace{-0.1 cm}}c@{\hspace{-0.4 cm}}c@{\hspace{-0.1 cm}}c@{\hspace{-0.4 cm}}c}
    &\subfloat[]{\label{fig.1}}
    &&
    \subfloat[]{\label{fig.2}}
    &&
    \subfloat[]{\label{fig.3}}
    &&
    \subfloat[]{\label{fig.4}}  
    \\[0.cm]
    \begin{tabular}{c}
      \vspace{.25cm}
      \small{\rotatebox{90}{$\log \, (\norm{P(t) - \Lambda})_2)$}}
    \end{tabular}
    &
    \begin{tabular}{c}
      \includegraphics[width=0.2\textwidth]{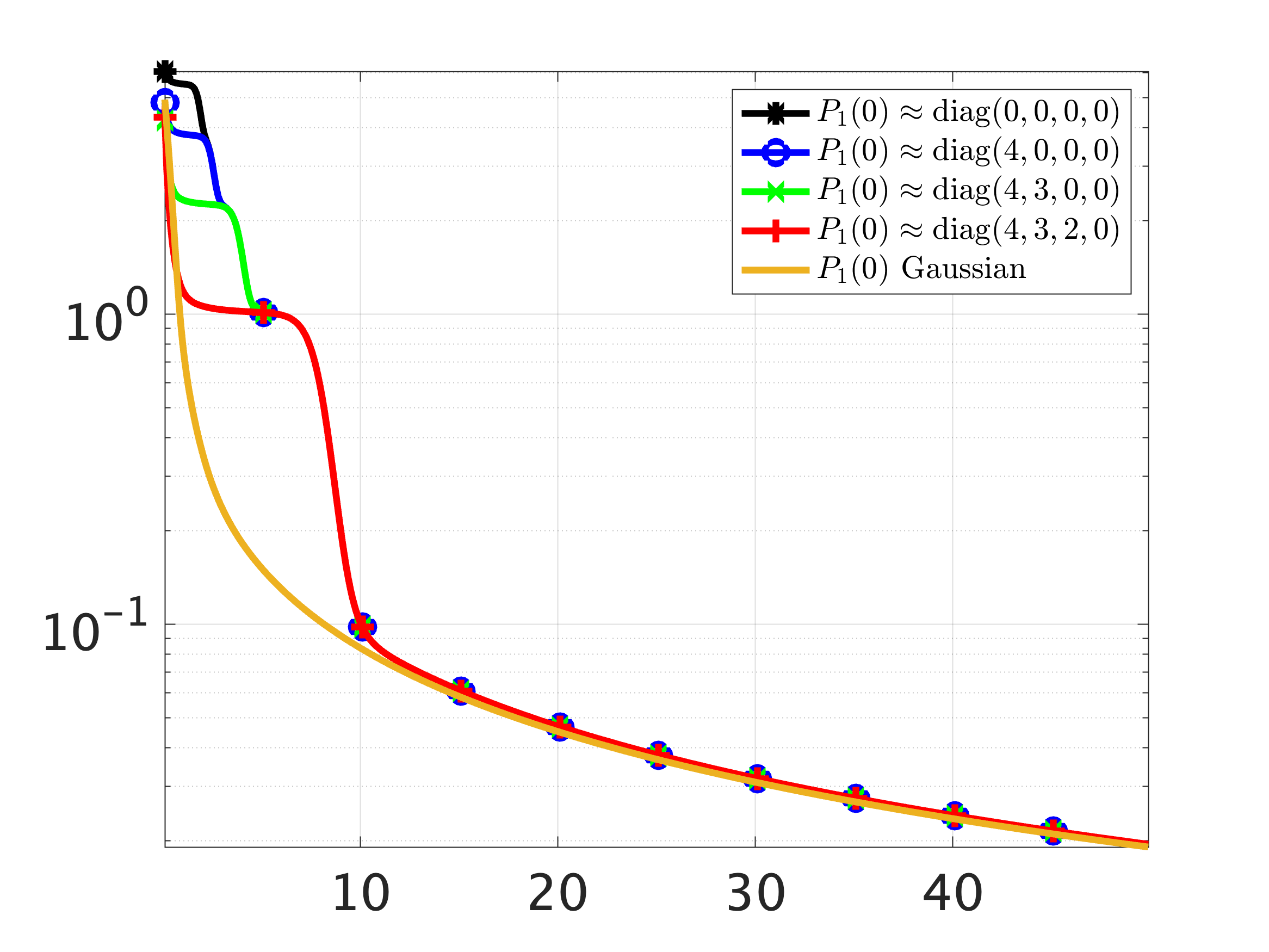}
      \\[-0.2cm]  {$t$}
    \end{tabular}
    &
    \begin{tabular}{c}
      \vspace{.25cm}
      \small{\rotatebox{90}{$\log \, (\norm{P_1(t)-\Lambda_1}_2)$}}
    \end{tabular}
    &
    \begin{tabular}{c}
      \includegraphics[width=0.2\textwidth]{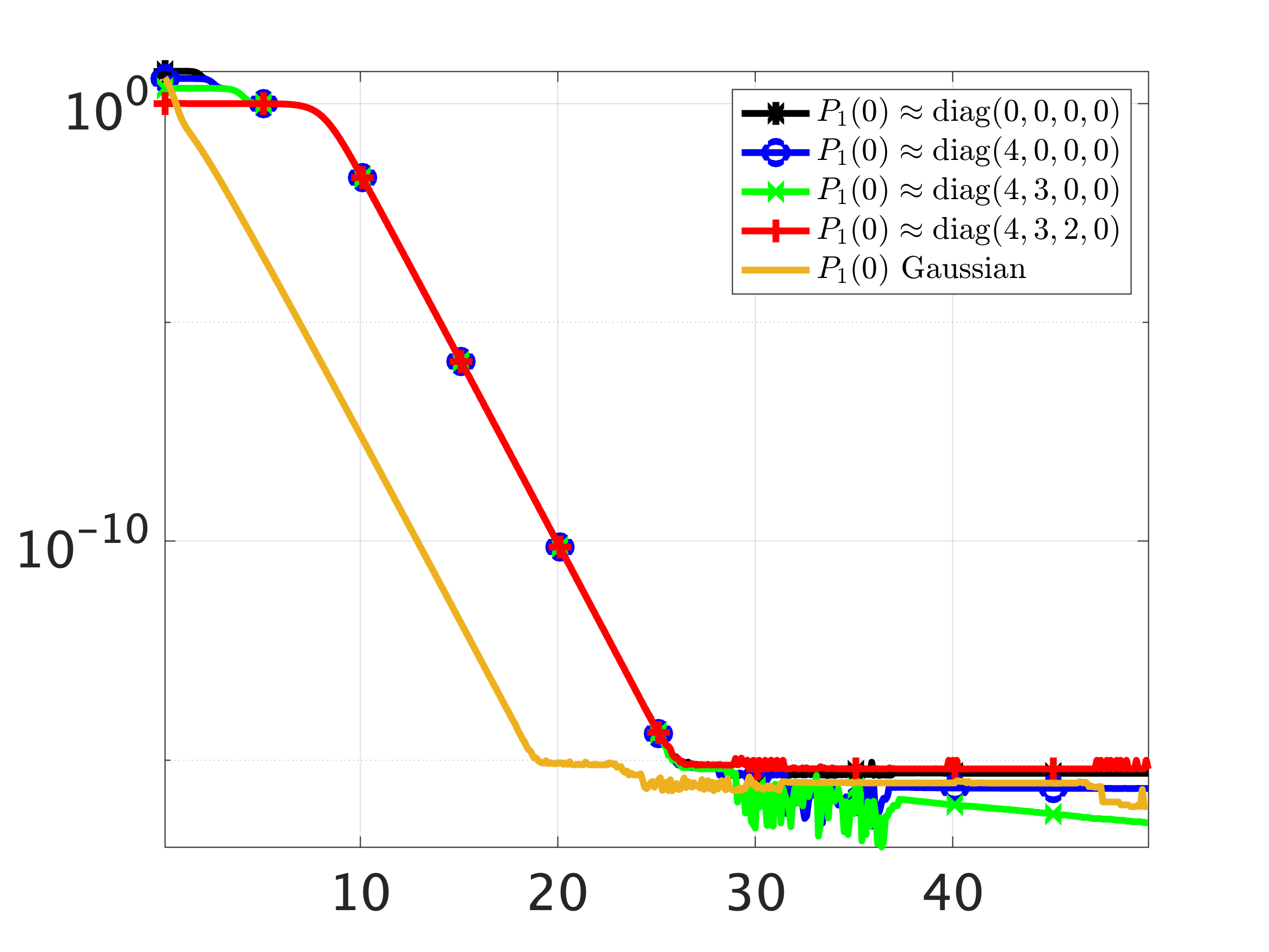}
      \\[-0.2 cm] {$t$}
    \end{tabular}
    &
    \begin{tabular}{c}
        \vspace{.25cm}
        \small{\rotatebox{90}{$\log \, (\norm{P_0(t)}_2$)}}
      \end{tabular}
      &
      \begin{tabular}{c}
        \includegraphics[width=0.2\textwidth]{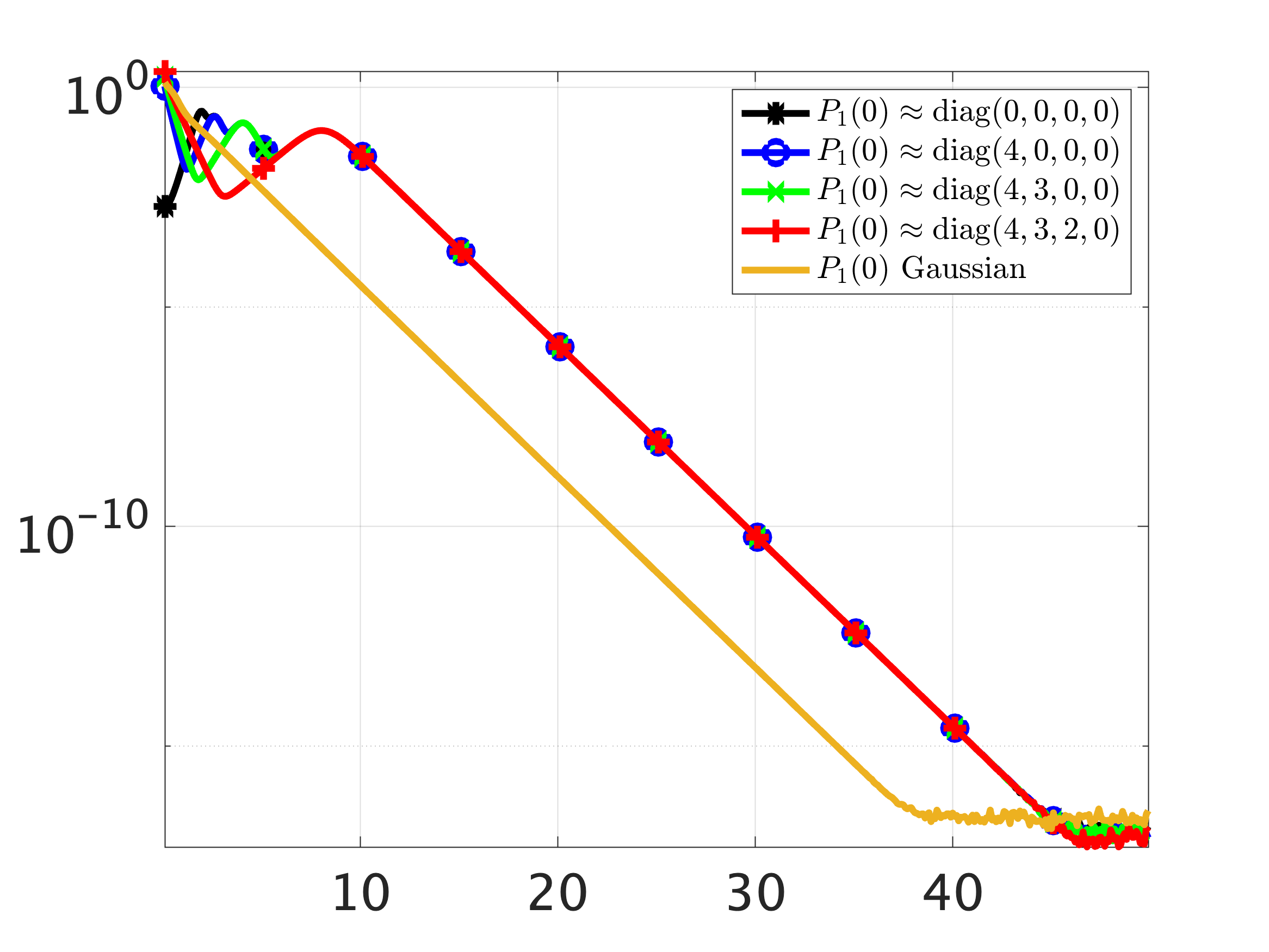}
        \\[-0.2 cm] { $t$}
      \end{tabular}
    &
    \begin{tabular}{c}
      \vspace{.25cm}
      \small{\rotatebox{90}{$\log \, (\norm{P_2(t)}_2)$}}
    \end{tabular}
    &
    \begin{tabular}{c}
      \includegraphics[width=0.2\textwidth]{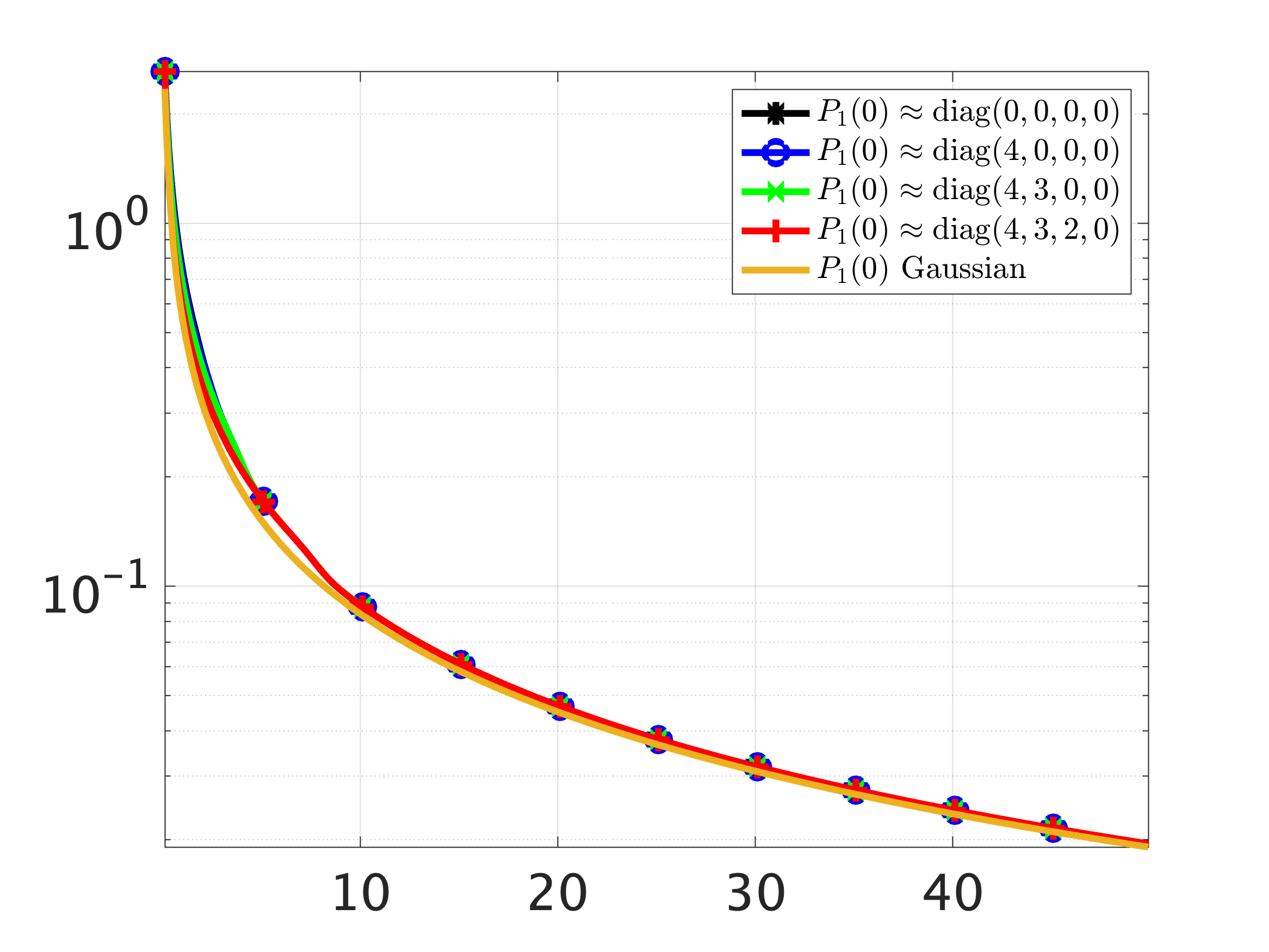}
      \\[-0.2 cm]  {$t$}
    \end{tabular} 
  \end{tabular}
  \vspace{-0.1cm}
  \caption{\tcb{Results obtained using simulation of gradient flow dynamics~\eqref{ode.gfd-X} for a problem with $n=10$, $r^\star=4$, $r=8$, $\Lambda_1=\diag(4,3,2,1)$, and $\Lambda_2 =0$. Colors mark trajectories with random initial conditions where $P_1 (0)$ is selected: around the equilibrium points $\bar{P}_1=\diag(0,0,0,0)$ (black), $\bar{P}_1=\diag(4,0,0,0)$ (blue),  $\bar{P}_1=\diag(4,3,0,0)$ (green), $\bar{P}_1=\diag(4,3,2,0)$ (red); with Gaussian distribution (yellow).}}
   \label{fig.example}
\end{figure*}

	\vspace*{-2ex}
	\section{Concluding remarks}\label{section-remarks}
        
                We have examined the gradient flow dynamics for over-parameterized symmetric low-rank matrix factorization problem under the most general conditions. Our proof is based on a novel nonlinear coordinate transformation that converts the original problem into a cascade connection of three subsystems. In spite of the lack of convexity, we proved that this system globally converges to the stable equilibrium point if and only if the initialization is \tcb{well-aligned with the optimal solution}. We used signal/noise decomposition to show that the subsystem associated with the signal component exponentially converges to the \tcb{target} low-rank matrix. \tcb{Our analysis also} reveals that the Schur complement associated with excess parameters vanishes with $O(1/t)$ rate, \tcb{thereby demonstrating that over-parameterization inevitably decelerates the algorithm.} 
               
               \tcb{Potential future directions include the extension of our results to (i) the asymmetric setup aimed at computing a low-rank factorization of a rectangular matrix $M\in \R^{m \times n}$; (ii) the matrix sensing problem where only some linear measurements of the target matrix are observed; and (iii) the optimization of structured problems using gradient flow dynamics~\cite{scott2017constrained,del2017solution}.}       
	
	\vspace*{-1ex}
	 {\renewcommand{\baselinestretch}{.97}

    }

        \appendix
        
 \vspace*{-2ex}
	\subsection{Characterization of equilibrium points}
	\label{app.equiPoints}

It is easy to verify that any $\bar{P}\in\cH$ is an equilibrium point of system~\eqref{ode.lift}. To show the converse, let the pair \tcb{($\bar{P} \succeq 0, \Lambda$)} make the right-hand side of~\eqref{ode.P1-P0-P2} vanish,	
	\begin{subequations}\label{eq.appEQP-temp}
		\begin{align}\label{eq.appEQP-temp1}
			0
			&\;=\;
			\bar{P}_1\Lambda_1+\Lambda_1\bar{P}_1 - 2(\bar{P}_1^2 \,+\, \bar{P}_0 \bar{P}_0^T)
			\\[0.cm]\label{eq.appEQP-temp2}
			0
			&\;=\;
			\Lambda_1\bar{P}_0
			\,-\,\bar{P}_0\Lambda_2
			\,-\,2(\bar{P}_1\bar{P}_0\,+\,\bar{P}_0\bar{P}_2)
			\\[0.cm]
			\label{eq.appEQP-temp3}
			0
			&\;=\;
			-\bar{P}_2\Lambda_2\,-\,\Lambda_2\bar{P}_2
			\,-\, 2(
			\bar{P}_0^T\bar{P}_0 + \bar{P}_2^2).
		\end{align}
	\end{subequations} 
{\color{black}
	 Since $P\succeq 0$, we have $\bar{P}_2\succeq 0$. To show that $\bar{P}_2=0$, let us assume that $\bar{P}_2$ has a positive eigenvalue $\mu> 0$ with the associated eigenvector $p$.  Pre and post multiplying~\eqref{eq.appEQP-temp3} by $p^T$ and $p$, respectively, and rearranging terms yields,
\begin{equation}
    	\mu p^T\Lambda_2 p 
	\, = \, 
	-\norm{\bar{P}_0 p}_2^2 
	\, - \,
	\mu^2 \norm{p}_2^2 
	\, < \,
	0.
    \end{equation}
This contradicts positive semi-definiteness of $\Lambda_2$ and implies that $\bar{P}_2 = 0$. Furthermore, substitution of $\bar{P}_2 = 0$ to~\eqref{eq.appEQP-temp3} gives $\bar{P}_0 = 0$ and equation~\eqref{eq.appEQP-temp1} simplifies to,}
	\begin{align}\label{eq.appEQP-barP1-simple}
		\bar{P}_1\Lambda_1+\Lambda_1\bar{P}_1
		\;=\;
		2\bar{P}_1^2.
	\end{align}
	Now, consider the eigenvalue decomposition of $\bar{P}_1$,
	$
	\tcb{\bar{P}_1
	=
	\bar{V}\bar{\Sigma} \bar{V}^T
	=
	\bar{V}_\mathbf{i}\bar{\Sigma}_\mathbf{i}\bar{V}_\mathbf{i}^T
	\neq
	0,}
	$
where $\bar{\Sigma}=\diag(\bar{\Sigma}_\mathbf{i},0)\in\R^{r^\star\times r^\star}$ with $\bar{\Sigma}_\mathbf{i}\in\R^{l\times l}$ containing the nonzero eigenvalues and, $\bar{V} \DefinedAs [ \, \bar{V}_\mathbf{i} ~~ \bar{V}_{\mathbf{o}} \, ]$ with $\bar{V}_\mathbf{i}\in \R^{r^\star \times l}$ containing the corresponding orthonormal eigenvectors as its columns. Equation~\eqref{eq.appEQP-barP1-simple}  yields
	\begin{subequations}
		\begin{align}\label{eq.equilibPointsTemp1}
			A\,\bar{\Sigma}
			\,+\,
			\bar{\Sigma}\,A
			\;=\;
			2\bar{\Sigma}^2
		\end{align}
		where 
		$
		A
		\DefinedAs
		\bar{V}^T\Lambda_1 \bar{V} .
		$

  If $l=r^\star$, i.e., if $\bar{\Sigma}$ is full rank,~\eqref{eq.equilibPointsTemp1} implies that $A$ is also diagonal and, hence, $A=\bar{\Sigma}$. From the definition of $A$ and $\bar{V}$, it follows that $\bar{P}_1=\Lambda_1$ is an equilibrium point.

  To address the case with $l\neq r^\star$, let \tcb{us partition $A$}
		\begin{align}\label{eq.equilibPointsTemp2}
			A
			\;=\;
			\tbt{A_\mathbf{i}}{A_\mathbf{a}}{A_\mathbf{a}^T}{A_\mathbf{o}}
		\end{align}
with $A_\mathbf{i}\in\R^{l\times l}$.
		{Using~\eqref{eq.equilibPointsTemp1}, we observe that
        \begin{align}
            \tbt{A_\mathbf{i}\bar{\Sigma}_i + \bar{\Sigma}_iA_\mathbf{i}}{\bar{\Sigma}_i{A_\mathbf{a}}}{A_\mathbf{a}^T\bar{\Sigma}_i}{0}\;=\; \tbt{2\bar{\Sigma}_i^2}{0}{0}{0}.
        \end{align}
        Thus, 
		$
		A_\mathbf{i}
		=
		\bar{\Sigma}_\mathbf{i}
		$ and
		$
		A_\mathbf{a}
		=
		0.
		$}
Substituting this to
		$
		A
		\DefinedAs
		\bar{V}^T\Lambda_1 \bar{V} 
		$ 
		yields
		{\color{black}
			\begin{align}\label{eq.appEQP-helper1}
			\Lambda_1
			&\;=\;
			\bar{V}\tbt{\bar{\Sigma}_\mathbf{i}}{0}{0}{A_\mathbf{o}}\bar{V}^T
   			\;=\;
			\bar{P}_1\,+\,
			\bar{V}_\mathbf{o}A_\mathbf{o}\bar{V}_\mathbf{o}^T 
		\end{align}
	}
Hence, 
		$
		\Lambda_1
		-
		\bar{P}_1
		=
		\bar{V}_\mathbf{o}A_\mathbf{o}\bar{V}_\mathbf{o}^T.
		$
\tc{black}{Since} $\bar{V}_\mathbf{i}$ and $\bar{V}_\mathbf{o}$ are mutually orthogonal, i.e., $\bar{V}_\mathbf{i}^T\bar{V}_\mathbf{o}=0$, we have
		$
		\bar{P}_1(\Lambda_1-\bar{P}_1)
		=
		0
		$
\tcb{and} $\cH$ is the set of equilibrium points of~\eqref{ode.lift}. 
	\end{subequations}
	
	\subsubsection{Proof of Lemma~\ref{lem.eqPointsSecCharac}}
	\tc{black}{First, we show that the existence of such a matrix $V$ and diagonal matrices $D_{\mathbf{i}}$ and $D_{\mathbf{o}}$ is sufficient for $\bar{P}\in\cH$. We observe that
    \begin{align*}
        \bar{P}_1 \Lambda_1 &\;=\; \bar{P}_1 (\bar{P}_1\,+\,\Lambda_1\,-\,\bar{P}_1) \;=\; \bar{P}_1^2 \,+\, \bar{P}_1 (\Lambda_1\,-\,\bar{P}_1) \\
        &\;=\; \bar{P}_1^2 \,+\, \bar{P}_1 V_{\mathbf{o}}D_{\mathbf{o}}V_{\mathbf{o}}^T\;\stackrel{(a)}{=}\;\bar{P}_1^2
    \end{align*}
    where $(a)$ follow from the fact that $V_{\mathbf{i}}$ and $V_{\mathbf{o}}$ are orthogonal. To show the necessity, first, let $A_\mathbf{o} = U_\mathbf{o} D_{\mathbf{o}} U_\mathbf{o}^T$ be the eigenvalue decomposition of $A_\mathbf{o}$. We can write~\eqref{eq.appEQP-helper1} as
    \begin{align*}
        \bar{V}_\mathbf{i}\bar{\Sigma}_\mathbf{i}\bar{V}_\mathbf{i}^T
			\,+\,\bar{V}_\mathbf{o}A_\mathbf{o}\bar{V}_\mathbf{o}^T \;=\;
    \bar{V}_\mathbf{i}\bar{\Sigma}_\mathbf{i}\bar{V}_\mathbf{i}^T
			\,+\,
			\bar{V}_\mathbf{o}U_\mathbf{o} D_{\mathbf{o}} U_\mathbf{o}^T\bar{V}_\mathbf{o}^T.
    \end{align*}
    Let us define
	\[
	V_{\mathbf{i}}
	\;\DefinedAs\;
	\bar{V}_\mathbf{i}
	,\;
	D_{\mathbf{i}}
	\;\DefinedAs\;
	\bar{\Sigma}_\mathbf{i},\;
	V_{\mathbf{o}}
	\;\DefinedAs\;
	\bar{V}_\mathbf{o} U_\mathbf{o}
	.
	\]\
We observe that $V_{\mathbf{o}}$ is a unitary matrix and orthogonal to $V_{\mathbf{i}}$. We will show that such a $(V_{\mathbf{i}},D_{\mathbf{i}},V_{\mathbf{o}},D_{\mathbf{o}})$ satisfies the conditions in Lemma~\ref{lem.eqPointsSecCharac}. Based on the definition, we have 
\begin{align*}
    \bar{P}_1
		\;=\;
		V_{\mathbf{i}}D_{\mathbf{i}}V_{\mathbf{i}}^T
		,\;\;
		\Lambda_1\,-\,\bar{P}_1
		\;=\;
		V_{\mathbf{o}}D_{\mathbf{o}}V_{\mathbf{o}}^T.
\end{align*}
We also know that $\eig(D_{\mathbf{i}})\cup\eig(A_{\mathbf{o}})=\eig(\Lambda_1)$, and $\eig(D_{\mathbf{o}})=\eig(A_{\mathbf{o}})$, which proves the necessity.
Finally, suppose that the eigenvalues of $\Lambda_1$ are all distinct. We know that 
\begin{align*}
    \Lambda_1 = \obt{V_{\mathbf{i}}}{V_{\mathbf{o}}}\tbt{D_{\mathbf{i}}}{0}{0}{D_{\mathbf{o}}}\tbo{V_{\mathbf{i}}^T}{V_{\mathbf{o}}^T}.
\end{align*}
In this case, both $\Lambda_1$ and $\diag(D_{\mathbf{i}},D_{\mathbf{o}})$ are diagonal matrices with distinct values. Hence, the matrix $\obt{V_{\mathbf{i}}}{V_{\mathbf{o}}}$ has to be a permutation matrix as the eigenvalues are unique. This implies that $\diag(D_{\mathbf{i}},D_{\mathbf{o}})$ can only be a permuted version of $\Lambda_1$. Thus, there is a one-to-one correspondence between the equilibrium points and the $2^{r^\star}$ subsets of positive eigenvalues. This completes the proof.} 
	
	\vspace*{-2ex}
	\subsection{Local convergence results}
	\label{app.localstabResults}
	
	\subsubsection{\tcb{Proof of Proposition~\ref{prop.localStability}}}
		We consider the objective function in optimization problem~\eqref{eq.mainX} as a Lyapunov function candidate,
		\[
		V_F(P)
		\;=\;
		( 1 / 4 ) \norm{P \, - \, \Lambda}_F^2.
		\]	
		The derivative of $V_F$ along the trajectories of~\eqref{ode.lift} satisfies
		\tc{black}{
			\begin{align*}  
				\dot{V}_F
				&\;=\;
				( 1 / 4 ) \,\trace \, [(P \, - \, \Lambda) \dot{P}] + ( 1 / 4 ) \,\trace \, [\dot{P} (P \, - \, \Lambda)]\\
				&\;\stackrel{(a)}{=}\;
				( 1 / 2 ) \,\trace \, [(P \, - \, \Lambda) ( ( \Lambda \, - \, P )P
				\, + \,
				P (\Lambda \, - \, P ) ) ]\\
				&\;\stackrel{(b)}{=}\; 
				-\trace \, [ (P \, - \, \Lambda )^2 P ] 
				\;=\; 
				-\norm{(P \, - \, \Lambda)P^{\frac{1}{2}}}_F^2
			\end{align*}
			where $(a)$ and $(b)$ follow by the cyclic property of the matrix trace. Stability} of $\bar{P}= \diag(\Lambda_1,0)$ in the sense of Lyapunov follows from $\dot{V}_F\le 0$ for all $P\succeq 0$. To show local asymptotic stability, we note that $\dot{V}_F (P) = 0$ \tcb{if and only if} $(P - \Lambda) P^{\frac{1}{2}}=0$. \tcb{Thus,} $(P - \Lambda) P = 0$ and $\dot{V}_F(P) = 0$ only \tcb{at equilibria} of system~\eqref{ode.lift}. 
		\tcb{Since $\diag(\Lambda_1,0)$ is an isolated equilibrium point,}	
		$\dot{V}_F(P)$ is negative over the open neighborhood of $\diag (\Lambda_1,0)$,
		$
		\{P\;|\,\norm{P - \diag(\Lambda_1,0)}_2 < \lambda_{r^\star} \},
		$
		and \tcb{$\dot{V}_F(P) = 0$ if and only if $P=\diag (\Lambda_1,0)$.} Hence, $\bar{P}=\diag (\Lambda_1,0)$ is a locally asymptotically stable equilibrium point of system~\eqref{ode.lift}.
	
	\vsp
	
	\subsubsection{\tcb{Proof of Lemma~\ref{lem.sym-noise-decrease}}}

		Let \tcb{$(\lambda (t), w(t)) \in (\R_+, \R^{n-r^\star})$ be the principal eigenpair} of the matrix $P_2(t)$
		with $w^T (t) w (t) = 1$. Since
		\[
		\dot{w}^T (t) P_2 (t) w (t) + w^T (t) P_2 (t) \dot{w} (t) 
		=  
		\lambda (t) 
		\,
		\dfrac{\mrd (w^T (t) w(t))}{\mrd t}
		= 
		0
		\]
		the derivative of \tcb{$V_N (t) \DefinedAs \lambda (t) = w^T (t) P_2 (t) w (t)$} along the solutions of~\eqref{ode.lift} satisfies
		\begin{align*}
			\dot{V}_N
			&\;=\;
			w^T\dot{P_2} w
			\, + \, 
			\tcb{\dot{w}^T P_2 w
				\, + \,
				w^T P_2 \dot{w} 
				\; = \;
				w^T\dot{P_2} w
			}
			\\[0.cm]
			&
			\;=\;
			-w^T( P_2\Lambda_2 \,+\, \Lambda_2P_2
			\,+\, 
			2(
			P_0^TP_0 + P_2^2)) w
			\\[0.cm]
			&\;=\;
			-2V_N^2 \,-\, 2 \norm{\Lambda_2^{\frac{1}{2}}w}_2^2 \,V_N\,-\, 2\norm{P_0 w}_2^2
			\; \le \,
			-2V_N^2.
		\end{align*}
		For \tcb{$x (t_0) > 0$,} $x(t)=2x(t_0)/(1+t - t_0)$ \tcb{solves} $\dot{x}=-2x^2$ and the result follows from comparison principle.

	\vspace*{-2ex}
	\subsection{Global convergence results}
	\label{app.globalstabProp}
	
		\subsubsection{Proof of Proposition~\ref{prop.changeVariables}}
	
We can write
	\begin{align*}
		\dot{H}_1
		&\;=\;
		\mrd ( P_1^{-1} ) / \mrd t
		\\[0. cm]
		&\;=\;
		-H_1\dot{P}_1 H_1
		\\[0. cm]
		&\;=\;
		-H_1\left(P_1\Lambda_1+\Lambda_1P_1 - 2(P_1^2 \,+\, P_0 P_0^T)\right)H_1
		\\[0. cm]
		&\;=\;
		-H_1\left(H_1^{-1}\Lambda_1+\Lambda_1H_1^{-1} - 2(H_1^{-2}\,+\, P_0 P_0^T)\right)H_1
		\\[0. cm]
		&\;=\;
		-\Lambda_1 H_1 \,-\, H_1\Lambda_1
		\,+\, 2 I \,+\, 2 H_1P_0P_0^TH_1
		\\[0. cm]
		&\;=\;
		-\Lambda_1 H_1 \,-\, H_1\Lambda_1
		\,+\, 2 I \,+\, 2 H_0 H_0^T
	\end{align*}
where the last equality follows from $H_1P_0=H_0$. Similarly, for the $H_0$-dynamics we have
	\begin{align*}
	\dot{H_0}
		&\;=\;
		\mrd (H_1P_0)/\mrd t
		\\[0. cm]
		&\;=\;
		H_1\dot{P}_0\,+\,\dot{H_1}P_0
		\\[0. cm]
		&\;=\;
		H_1\left(\Lambda_1P_0
		\,-\,
		P_0\Lambda_2
		\,-\,2(P_1P_0\,+\,P_0P_2)\right)
		\\[0. cm]
		&\;~\;+
		\left(-\Lambda_1 H_1 \,-\, H_1\Lambda_1
		\,+\, 2 I \,+\, 2 H_0 H_0^T\right)P_0
		\\[0. cm]
		&\;=\;
		-\Lambda_1H_1P_0
		\,-\,
		H_1 P_0\Lambda_2
		\,-\,
		2(H_1P_0P_2-H_0H_0^TP_0)
		\\[0. cm]
		&\;=\;
		-\Lambda_1H_0
		\,-\,
		H_0\Lambda_2
		\,-\, 2H_0(P_2\,-\,P_0^TP_1^{-1}P_0)
		\\[0. cm]
		&\;=\;
		-\Lambda_1H_0
		\,-\,
		H_0\Lambda_2
		\,-\, 2H_0H_2.	
	\end{align*}
	Finally, the $H_2$-dynamics are given by
	\begin{align*}
		\dot{H}_2
		&\;=\;
		\mrd(P_2\,-\,P_0^TH_1 P_0)/\mrd t
		\\[0. cm]
		&\;=\;
		\dot{P}_2
		\,-\, 
		\dot{P}_0^T H_1 P_0
		\,-\, 
		P_0^T H_1 \dot{P}_0
		\,-\, 
		P_0^T \dot{H}_1 P_0
		\\[0. cm]
		&\;=\;
		-P_2\Lambda_2
		\,-\,
		\Lambda_2P_2
		-2(P_0^TP_0 + P_2^2)
		\\[0. cm]
		&\;~~\;\,-\,
		(\Lambda_1P_0
		\,-\,
		P_0\Lambda_2
		\,-\,2(P_1P_0\,+\,P_0P_2))^T H_1P_0
		\\[0.cm]
		&\;~~\;\,-\,
		P_0^T H_1 (\Lambda_1P_0
		\,-\,
		P_0\Lambda_2
		\,-\,2(P_1P_0\,+\,P_0P_2))
		\\[0. cm]
		&\;~~\;\,-\,
		P_0^T(-\Lambda_1 H_1 \,-\, H_1\Lambda_1
		\,+\, 2 I \,+\, 2 H_0 H_0^T)P_0
		\\[0. cm]
		&\;=\;
		-\Lambda_2(P_2\,-\,P_0^TH_1 P_0)\,-\,(P_2\,-\,P_0^TH_1 P_0)\Lambda_2
		\\[0. cm]
		& \; \phantom{=} \;
		~ -2 (P_2\,-\,P_0^TH_1 P_0)^2
		\\[0. cm]
		&
		\;=\;
		-\Lambda_2 H_2 \,-\, H_2 \Lambda_2 \,-\, 2 H_2^2
	\end{align*}
which completes the proof.
		
	\subsubsection{Proof of Theorem~\ref{thm.analytical-sol-symmetric-H1}}
	
	We first present a technical result.
	The next lemma establishes the exponential decay of $H_0$.
	\begin{mylem}\label{lem.H2convergence-sym}
		For the matrix $H_0$ governed by~\eqref{ode.cascade-full}, the derivative of the spectral norm satisfies
	   \[
		\dfrac{\mrd  \norm{H_0}_2}{\mrd t}
		\;\le\;
		-\lambda_{r^\star} \norm{H_0}_2.
		\]
	\end{mylem}
	\begin{proof}
		Let $u(t)$ and $v(t)$ be the principal left and right singular vectors of $H_0(t)$ with $\norm{u(t)}=\norm{v(t)}=1$. Then, the spectral norm $\norm{H_0}_2 =u^T H_0 v$ satisfies
		\begin{align*}
			\dfrac{\mrd  \norm{H_0}_2}{\mrd t}
			&\;=\;
			u^T\dot{H}_0v
			\;=\;
			u^T(- \Lambda_1H_0 
			\,-\,
			H_0 \left(2P/P_1\,+\, \Lambda_2\right))v
			\\[0.cm]
			&\;=\;
			-\left(
			u^T\Lambda_1u
			\,+\,
			v^T \left(2P/P_1\,+\, \Lambda_2\right)v
			\right)\,\norm{H_0}_2
			\\[0.cm]
			&\;\le\;
			-\lambda_{r^\star} \norm{H_0}_2
		\end{align*}
		\tcb{Here, the first equality is a well-known property of the derivative of singular vectors, the inequality follows from the fact that $u^T\Lambda_1u\geq\lambda_{r^\star}$ and $2P/P_1 + \Lambda_2\succeq 0.$} This proves the first inequality in \tcb{Lemma~\ref{lem.H2convergence-sym}.}
	\end{proof}
	
	We are now ready to prove Theorem~\ref{thm.analytical-sol-symmetric-H1}. 
	
	The shifted matrix variable $\tilde{H}_1=H_1-\Lambda_1^{-1}$ brings equation~\eqref{ode.cascade-full} for $H_1$ to
	\[
		\dot{\tilde{H}}_1
		=
		-\Lambda_1 \tilde{H}_1 - \tilde{H}_1\Lambda_1 + 2 H_0H_0^T.
	\]
	This system is linear in $\tilde{H}_1$, and it is driven by the exogenous input $2 H_0 H_0^T$. The state transition operator is determined by $M\rightarrow\mre^{-\Lambda_1t}M
	\mre^{-\Lambda_1 t}$, the variation of constants formula yields
	\[
		\tilde{H}_1(t)
		\;=\;
		\mre ^{-\Lambda_1t}
		\tilde{H}_1(0)
		\mre ^{-\Lambda_1t}
		\,+\,
		2 \Phi(t)
	\]
	where the forced response is determined by 
	\be
	\Phi(t)
	\;\DefinedAs\;
	\int_{0}^{t}
	\mre ^{-\Lambda_1(t-\tau)}
	H_0(\tau) H_0^T (\tau)
	\mre ^{-\Lambda_1(t-\tau)}
	\,\mrd \tau.
	\label{eq.Phi}
	\ee
	Substituting $\tilde{H}_1 = H_1-\Lambda_1^{-1}$ in the above equation yields the expression in Theorem~\ref{thm.analytical-sol-symmetric-H1} for $H_1(t)$. 
 The norm of this forced response can be bounded by	
 \begin{align}
        \nonumber
		\norm{\Phi(t)}_2
		&\;\overset{(a)}{\le}\;
		\int_{0}^{t}
		\norm{\mre ^{-\Lambda_1(t-\tau)}
			H_0(\tau)H_0(\tau)^T
			\mre ^{-\Lambda_1(t-\tau)}
		}_2
		\,\mrd \tau
		\\[0. cm]\nonumber
		&\;\overset{(b)}{\le}\;
		\int_{0}^{t}
		\norm{\mre ^{-\Lambda_1(t-\tau)}}_2^2
		\,	\norm{H_0(\tau)}_2^2
		\,\mrd \tau
		\\[0. cm]
		&\;=\;
		\int_{0}^{t}
		\mre ^{-2\lambda_{r^\star}(t-\tau)}
		\,	\norm{H_0(\tau)}_2^2
		\,\mrd \tau
		\label{eq.thm-analsol-sym-temp1-proof}
	\end{align}
	where $(a)$ follows from the triangle inequality and $(b)$ follows from the sub-multiplicative property of the spectral norm.

	To complete the proof, we note that if $H_1(t)\succ0$ then $P_1(t)=H_1^{-1}$ and $P_0=H_1^{-1} H_0$. In addition, the Schur complement $P/P_1$ exists which together with the exponential decay rate of $H_0(t)$ established in Lemma~\ref{lem.H2convergence-sym} yields
	\[
	\norm{H_0(t)}_2 \;\le\; \mre^{-\lambda_{r^\star}t}\norm{H_0(0)}_2.
	\]
	By combining this inequality with~\eqref{eq.thm-analsol-sym-temp1-proof}, we obtain
	\begin{align*}
		\norm{\Phi(t)}_2
		&\;\le\;
		\int_{0}^{t}
		\mre ^{-2\lambda_{r^\star}t}
		\,\mrd \tau
		\norm{H_0(0)}_2^2
		\;\le\;
		t\,\mre^{-2\lambda_{r^\star}t}
		\norm{H_0(0)}_2^2.
	\end{align*}
	\tcb{Finally, to derive the upper bound on $\norm{\tilde{H}_1(t)}_2$, we write} 	
	\begin{align*}
		\norm{\tilde{H}_1(t)}_2
		&\;\stackrel{(a)}{\leq}\;
		\norm{\mre ^{-\Lambda_1t}
			\tilde{H}_1(0)
			\mre ^{-\Lambda_1t}}_2
		\,+\,
		2\norm{\Phi(t)}_2
		\\[0.cm]
		&\;\stackrel{(b)}{\leq}\;
		\tcb{\left(
		\norm{
			\tilde{H}_1(0)
		}_2
		\,+\,
		2\,t\,\norm{H_0(0)}_2^2\right)\mre^{-2\lambda_{r^\star}t}}
	\end{align*}
\tcb{where $(a)$ follow from triangle inequality and $(b)$ follows from combining the above aforementioned facts.} This completes the proof of Theorem~\ref{thm.analytical-sol-symmetric-H1}.
	
	\subsubsection{Proof of Theorem~\ref{thm.conv-original-sym}}
	
	We first present a lemma that we use to establish an upper bound on the error $\norm{P_1-\Lambda_1}_2$ as a function of  $\norm{P_1^{-1}-\Lambda_1^{-1}}_2$.
	
	\begin{mylem}
		\label{lem.bound-relation-direct-inverse}
		Let the matrices $A\succ 0$ and $B$ be such that
		\[
		\norm{A - B}_2
		\;<\;
		\sigma \DefinedAs \sigma_{\min}(A).
		\]
		Then, the matrix $B$ is invertible, and it satisfies
		\begin{align*}
			\norm{A^{-1}\,-\,B^{-1}}_2
			\;\le\;
			\dfrac{\norm{A\,-\,B}_2}{\sigma\,(\sigma\,-\,\norm{A\,-\,B}_2)}.
		\end{align*}
	\end{mylem}
	\begin{proof}
		For $Q\DefinedAs I-BA^{-1} = (A- B)A^{-1}$, we have
		\begin{align}\label{eq.proof-lem-small-gain-temp1}
			\norm{Q}_2 
			&\;\le\;
			\norm{A\,-\, B}_2  \, \norm{A^{-1}}_2 
			\;=\;\norm{A\,-\, B}_2/\sigma <1.
		\end{align}
		Thus, the matrix $B A^{-1}=I-Q$ and consequently $B$ are both invertible by the small gain theorem~\cite{desoer2009feedback}. In addition,
		\begin{align*}
			B^{-1}\,-\,A^{-1}
			&\;=\;
			A^{-1}(AB^{-1}\,-\,I)
			\\[0. cm]
			&\;=\;
			A^{-1}((I\,-\,Q)^{-1}\,-\,I)
			\;=\;
			A^{-1}\sum_{k\,=\,1}^{\infty}Q^{k}.
		\end{align*}
		Thus, we can use the triangle inequality and the sub-multiplicative property of the spectral norm to obtain
		\begin{multline*}
			\norm{B^{-1}\,-\,A^{-1}}_2
			\;\le\;
			\sigma^{-1}
			\sum_{k\,=\,1}^{\infty}\norm{Q^{k}}_2
			\\[-0.15cm]
			\;\overset{(a)}{\le}\;
			\sigma^{-1}
			\sum_{k\,=\,1}^{\infty}\sigma^{-k}\norm{A\,-\, B}_2^k
			\;=\;
			\dfrac{\norm{A\,-\,B}_2}{\sigma\,(\sigma\,-\,\norm{A\,-\,B}_2)}
		\end{multline*}
		where $(a)$ follows from~\eqref{eq.proof-lem-small-gain-temp1}. This completes the proof.
	\end{proof}
	
{\color{black}

To prove the upper bound for $\norm{P_1(t)-\Lambda_1}_2$, we note that
the function $l(t)$ can be written in the $H$-coordinates as
\begin{align*}
	l(t)
	\;=\;&
	\left(
	\norm{P_1^{-1}(0) \,-\, \Lambda_1^{-1}}_2
	\,+\,
	2\,t\, \norm{P_1^{-1}(0)P_0(0)}_2^2
	\right)\mre^{-2\lambda_{r^\star}t} \\
	\;=\;& \left(
	\norm{
		\tilde{H}_1(0)
	}_2
	\,+\,
	2\,t\,\norm{H_0(0)}_2^2\right)\mre^{-2\lambda_{r^\star}t}.
\end{align*}
Thus, we have
	\begin{multline}\label{eq.proofTheoremMainHelper1}
		\norm{P_1^{-1}(t) - \Lambda_1^{-1}}_2
		\, = \,
		\norm{\tilde{H}_1(t)}_2
		\, \le \,
		l(t)
		\\[0. cm]
		\le \,
		1/(2\lambda_1)
		\, < \,
		1/\lambda_1
		\, = \,
		\sigma_{\min}(\Lambda_1^{-1})
	\end{multline}
where the first inequality follows from Theorem~\ref{thm.analytical-sol-symmetric-H1} and the second inequality holds by assumption as stated in Theorem~\ref{thm.conv-original-sym}.
The bound in~\eqref{eq.proofTheoremMainHelper1} allows us to apply Lemma~\ref{lem.bound-relation-direct-inverse} with $A\DefinedAs \Lambda_1^{-1}$ and $B\DefinedAs P_1^{-1}(t)$ to obtain
	\begin{align}\notag
		\norm{P_1(t) - \Lambda_1}_2
		&\;\le\;
		\dfrac{\norm{P_1^{-1}(t) - \Lambda_1^{-1}}_2}{\lambda_1^{-1}(\lambda_1^{-1}\,-\, \norm{P_1^{-1}(t) - \Lambda_1^{-1}}_2)}
		\\[0.2cm]
		&\;\le\;
		2\lambda_1^2\norm{P_1^{-1}(t) - \Lambda_1^{-1}}_2
		\;\le\;
		2\,\lambda_1^2\,l(t)
		\label{upperbound.p1}
	\end{align}
which establishes the desired upper bound for $\norm{P_1(t)-\Lambda_1}_2$.
	
Furthermore, the upper bound on 
	$
		\norm{P_0(t)}_2
	$
is given by	
	\begin{align*}
		\norm{P_0(t)}_2
		&\;\stackrel{(a)}{\le}\;
		\norm{P_1(t)}_2\,\norm{P_1^{-1}(t)P_0(t)}_2
		\\[0. cm]
		&\;\stackrel{(b)}{\le}\;
		\left(\norm{\Lambda_1}_2 \,+\, \norm{P_1(t)-\Lambda_1}_2\right)\norm{P_1^{-1}(t)P_0(t)}_2
		\\[0. cm]
		&\;\stackrel{(c)}{\le}\;
		\left(\lambda_1 \,+\, 2\,\lambda_1^2 \, l(t) 
		\right)
		\norm{P_1^{-1}(0)P_0(0)}_2\,
		\mre^{-\lambda_{r^\star} t}.
	\end{align*}
	where $(a)$ follows from the sub-multiplicative property of the spectral norm,  $(b)$ follows from triangle inequality, and $(c)$ follows from combining equation~\eqref{upperbound.p1} and the fact that
    \begin{align*}
        \norm{P_1^{-1}(t)P_0(t)}_2 \;=\; \norm{H_0(t)}_2\;\leq\;\mre^{-\lambda_{r^\star}t}\norm{H_0(0)}_2
    \end{align*}
    established in Theorem~\ref{thm.analytical-sol-symmetric-H1}.
    
    Finally, the inequality $P_1(t)\succ 0$ can be verified by noting that that $P_1^{-1}(t)=H_1(t)\succ 0$ exists and is bounded according to Theorem~\ref{thm.analytical-sol-symmetric-H1}.
    	This completes the proof.}
	
	\vsp
	
        {\color{black}\subsubsection{Proof of Corollary~\ref{conv.Zcoor}} 
        \label{sec.C1}
        We begin by noting that the condition on $Z(0)$ is equivalent to the condition $P_1(0)\succ0$ in Theorem~\ref{thm.conv-original-sym}. Thus, the convergence results in Theorem~\ref{thm.conv-original-sym} hold. For some $c_1>0$, \tcb{Lemma}~\ref{lem.sym-noise-decrease} implies, 
        \[
        \norm{P_2(t)}_2 \, = \, \norm{Z_2(t)}_2^2 \, \le \, c_1/t
        \] 
and thus $\norm{Z_2(t)}_2\le \sqrt{c_1/t}$. Moreover, for $r=r^\star$ we can use the exponential convergence of $P_2(t)$ established in Remark~\ref{exact.param.conv} to conclude that $\norm{Z_2(t)}_2\le c_2\mre^{-\lambda_{r^\star}t}$, for some $c_2>0$.

        We next prove the convergence of $Z_1(t)$. Applying the triangle inequality and the submultiplicity of the spectral norm to equation~\eqref{ode.Z1} yields 
        \begin{multline}\label{eq.triang.pf} 
            \norm{\dot{Z}_1}_F \;\leq\; \norm{\Lambda_1 \, - \, Z_1Z_1^T}_2\norm{Z_1}_F \; + \; \norm{Z_1 Z_2^T}_2\norm{Z_2}_F\\
            \;=\;\norm{\Lambda_1 \, - \, P_1}_2\norm{Z_1}_F \; + \; \norm{P_0}_2\norm{Z_2}_F.
\end{multline}
For any positive scalars $t_1<t_2$, we have  
\begin{multline}
\norm{\int_{t_1}^{t_2}\dot{Z}_1(s)\,\mrd s}_F
\;\leq\;
\int_{t_1}^{t_2}\norm{\dot{Z}_1(s)}_F \,\mrd s
\\
\leq\;
        \int_{t_1}^{t_2}\big(\norm{\Lambda_1 \, - \, P_1(s)}_2\norm{Z_1(s)}_F
        \\
        +\;
        \norm{P_0(s)}_2\norm{Z_2(s)}_F\big)\, \mrd s
\end{multline}
where the first inequality follows from the triangle inequality and the second follows from equation~\eqref{eq.triang.pf}. By Theorem~\ref{thm.conv-original-sym}, we have $\norm{\Lambda_1 - P_1(t)}_2 \le c_3\mre^{-2\lambda_{r^\star}t}$ and $\norm{P_0(t)}_2 \le c_4 \mre^{-\lambda_{r^\star}t}$ for scalars $c_3,c_4>0$. Moreover, using convergence of $P(t)=Z(t) Z^T(t)$ shown in Theorem~\ref{thm.conv-original-sym}, it is straightforward to verify that $\norm{Z(t)}_F^2 = \norm{Z_1(t)}_F^2 + \norm{Z_2(t)}_F^2 \leq c^2$ for a large enough $c>0$. Thus, for $t \gg 1$ we have
\begin{align}\label{conve.int.Z1}
\norm{\int_{t}^{\infty}\dot{Z}_1(s)\,\mrd s}_F
&\;\leq\;
c\int_{t}^{\infty}\big(\norm{\Lambda_1 \, - \, P_1(s)}_2
        \\
        &\;\;\qquad\qquad\qquad\quad+\;
        \norm{P_0(s)}_2\big)\, \mrd s\\
        &\;\leq\;c'\int_{t}^{\infty}\big(\mre^{-\lambda_{r^\star}s} + \mre^{-2\lambda_{r^\star}s}\big)\,\mrd s\\ 
        &\;\le\; c''\mre^{-\lambda_{r^\star}t}
\end{align}
where $c', c''>0$ are large enough scalars. Since $Z_1(t)$ is a bounded function with finite Frobenius norm of integral of $\dot{Z}$, it converges to a matrix $Z^{\star}_1$~\cite{folland1999real}. In addition, since $Z(0) Z^T(0)$ satisfies the condition of Theorem~\ref{thm.conv-original-sym}, $P_1(t)=Z_1(t) Z_1^T (t)$ converges to $\Lambda_1$ and thereby $Z^{\star}_1 Z^{\star T}_1 = \Lambda_1$. To prove the convergence rate of $Z_1(t)$, note that equation~\eqref{conve.int.Z1} yields
    \begin{align*}
        \norm{Z_1(t)-Z^{\star}_1}_F \;=\; \norm{\int_{t}^{\infty}\dot{Z}_1(s) \mrd s}_F \;\le\;c_5\mre^{-\lambda_{r^\star}t}.
    \end{align*}
    for some $c_5>0$.
    This completes the proof.
}
 
\subsection{Expanded decomposition}
\label{app.expanded-decomp}
Let us formalize the  strategy sketched in Remark~\ref{rem.cascade-expanded} by defining the recursive equations
\begin{align}\label{eq.recursive-system}
	U_{i+1}
	&\;=\;
	U_i/\hat{P}_{i,1}
	\;=\;
	\hat{P}_{i,2} \,-\, \hat{P}_{i,0}^T \hat{P}_{i,1}^{-1} \hat{P}_{i,0}
	\;\in\;\R^{m_{i}\times m_{i}}
\end{align}
for $i=1\ldots,k$ with the initialization $U_1=P\in\R^{n\times n}$. Here, the matrices $\hat{P}_{i,j}$ are obtained by the block decomposition
\begin{align}\label{eq.U-decomposition}
	U_i 
\;=\;
\tbt{\hat{P}_{i,1}}{\hat{P}_{i,0}}{\hat{P}_{i,0}^T}{\hat{P}_{i,2}}\;\in\;\R^{m_{i-1}\times m_{i-1}}
\end{align}
and they satisfy $\hat{P}_{i,1}\in\R^{n_i\times n_i}$, $\hat{P}_{i,0}\in\R^{n_i\times m_i}$, and $\hat{P}_{i,2}\in\R^{m_i\times m_i}$, where $n_i$ is the dimension of the $i$th principal subspace of  $\Lambda_1$ associated with the eigenvalue $\hat{\lambda}_i$ and $m_i = n-(n_1+\cdots n_{i})$ with $m_0\DefinedAs n$.
The change of variables 
\begin{equation}\label{eq.changevar-nonlin-expanded}
	\ba{rclcl}
	\hat{H}_{i,1}
	&\!\!\!\DefinedAs\!\!\!&
	\hat{P}_{i,1}^{-1}
	&\!\!\!\in\!\!\!&
	\R^{n_i\times n_i}
	\\[0.cm]
	\hat{H}_{i,0}
	&\!\!\! \DefinedAs \!\!\!&
	\hat{P}_{i,1}^{-1}\hat{P}_{i,0}
	&\!\!\!\in\!\!\!&
	\R^{n_i\times m_i}
	\\[0.cm]
	\hat{H}_{i,2}
	&\!\!\! \DefinedAs \!\!\!&
	U_{i+1}
	&\!\!\!\in\!\!\!&
	\R^{m_i\times m_i}
	\ea
\end{equation}
yields the cascade system~\eqref{ode.cascade-full--expanded}.

For the set of equations in~\eqref{eq.recursive-system} to be well defined, the matrices $\hat{P}_{i,1}$ need to be invertible. A necessary and sufficient condition for the invertibility of $\hat{P}_{i,1}$ is that $P_1$ is invertible.
\begin{myprop}\label{prop.schur-property-expanded-system}
	System~\eqref{eq.recursive-system} satisfies
	\begin{align}\label{eq.schur-prod-expanded}
		\det (P_1)
		&\;=\;
		\prod_{i=1}^{k}\det(\hat{P}_{i,1})
	\end{align}
Furthermore, we have
	\begin{align}\label{eq.tail-system-matching}
		H_2
		&
		\;=\;U_{k+1}
		\;=\;
		\hat{H}_{k,2} 
	\end{align}
where $H_2=P/P_1$ is the Schur complement of $P_1$ in $P$ .
\end{myprop}
Proposition~\ref{prop.schur-property-expanded-system} establishes the equivalence between the invertibility of matrices $\hat{P}_{i,1}$ and $P_1$.
To prove this result, we next present two key properties of the Schur complement.
\begin{mylem}\label{lem.lu-fact-schur}
	For any  symmetric block matrix
	\begin{align*}
		X
		\;=\;
		\tbt{X_1}{X_0}{X_0^T}{X}
	\end{align*}
	with invertible $X$, we  have the LU factorization
	\begin{align*}
		X
		\;=\;
		\tbt{I}{0}{X_0^TX_1^{-1}}{I}
		\tbt{X_1}{0}{0}{X/X_1}
		\tbt{I}{X_1^{-1}X_0}{0}{I}
	\end{align*}
\end{mylem}
\begin{mylem}\label{lem.schur-quotient-property}
	Consider the symmetric block matrices
	\begin{align*}
		X
		\;=\;
		\tbt{X_1}{X_0}{X_0^T}{X_2}
		,\quad
		X_1
		\;\DefinedAs\;
		\tbt{\bar{X}_1}{\bar{X}_0}{\bar{X}_0^T}{\bar{X}_2}
	\end{align*}
	For invertible $X_1$,  we have the identity
	\[
	X/X_1
	\;=\;
	Y/ Y_1
	\]	
	where
	$
	Y
	\DefinedAs
	\tbt{Y_1}{Y_0}{Y_0^T}{Y_2}
	=
	X/\bar{X}_1
	$
	is the Schur complement  of $\bar{X}_1$ in $X$ and its block dimensions are conformable \mbox{with $X$.}
\end{mylem}
\begin{proof}
	We first show that $Y_1=	X_1/\bar{X}_1$. We have	
	\begin{align*}
		Y
		&\;=\;
		\tbt{Y_1}{Y_0}{Y_0^T}{Y_2}
		\;\DefinedAs\;X/\bar{X}_1
		\\[0.15 cm]
		&\;=\;
		\tbt{\bar{X}_2}{S^T}{S}{X_2}-\tbo{\bar{X}_0^T}{R^T} \bar{X}_1^{-1}\obt{\bar{X}_0}{R}
		\\[0.15 cm]
		&\;=\;
		\tbt
		{\bar{X}_2   \,-\,   \bar{X}_0^T\bar{X}_1^{-1}\bar{X}_0}
		{S^T   \,-\,   \bar{X}_0^T\bar{X}_1^{-1}R}
		{S   \,-\,   R^T\bar{X}_1^{-1}\bar{X}_0}
		{X_2   \,-\,   R^T\bar{X}_1^{-1}R}
	\end{align*}
	where $X_0^T=\obt{R^T}{S^T}$. Matching the 11-block in the above equation yields $Y_1=	X_1/\bar{X}_1$.
	
	We can now write
	\[
	X/X_1
	\;=\;
	(X/\bar{X}_1)/ (X_1/\bar{X}_1)=Y/ Y_1
	\]
	where the first equality is the quotient identity~\cite{crahay69} and the second one follows from the definition of $Y$. This completes the proof.
\end{proof}  

\subsubsection{Proof of Proposition~\ref{prop.schur-property-expanded-system}}
We use induction to show that
\begin{align}\label{eq.recursiveShur}
	P/\hat{P}_i
	\;=\;
	U_{i+1}
\end{align}
where $\hat{P}_i\in\R^{(n-m_i)\times(n-m_i)}$ is the 11-block of $P$. Equation~\eqref{eq.recursiveShur} for $i=1$ follows from~\eqref{eq.recursive-system} and the fact that $\hat{P}_1=\hat{P}_{1,1}$. To prove the case $i+1$, we write
\begin{align*}
	P/\hat{P}_{i+1}
	&\;=\;
	(P/\hat{P}_{i})
	/((P/\hat{P}_{i}))_{11}
	\\[0.cm]
	&\;=\;
	U_{i+1}
	/(U_{i+1})_{11}
	\\[0.cm]
	&\;=\;
	U_{i+1}/\hat{P}_{{i+1},1}
	\;=\; 
	U_{i+2}.
\end{align*}
where we use $(\cdot)_{11}$ to denote the 11-block of the size $n_{i+1}$.
Here, the first equality follows from Lemma~\ref{lem.schur-quotient-property} with $X=P$, $X_1=\hat{P}_{i+1}$, and $\bar{X}_1=\hat{P}_{i}$, the second equality is the inductive hypothesis, the third equality follows from~\eqref{eq.U-decomposition}, and the last equality follows from~\eqref{eq.recursive-system}. This completes the proof of~\eqref{eq.recursiveShur}. This equation for $i=k$ allows us to write
\[
H_2
\;=\;
P/P_1
\;=\;
P/\hat{P}_k
\;=\;
U_{k+1}
\overset{(a)}{\;=\;}
\hat{H}_{k,2}
\;\AsDefined\;
\hat{H}_2.
\]
where $(a)$ follows from the definition in~\eqref{eq.changevar-nonlin-expanded}. This completes the proof of~\eqref{eq.tail-system-matching}

To prove~\eqref{eq.schur-prod-expanded}, we can recursively apply Lemma~\ref{lem.lu-fact-schur} to the matrices $U_i$ starting from $U_1=P$ to form an LU decomposition  $P=\Phi Y \Psi$, where $\Phi$ and $\Psi$ are lower and upper diagonal matrices with 1 on the main diagonal, respectively, and 
\[
Y=\diag(\hat{P}_{1,1},\,\cdots,\,\hat{P}_{k,1},\,\hat{H}_{k,2})
\]
Now, since $\det(\Phi)=\det(\Psi)=1$, we obtain that
\begin{multline*}
\det(P)
\;=\;
\det(Y)
\\[0.cm] 
=\;
\det(\hat{H}_{k,2})\prod_{i=1}^{k}\det(\hat{P}_{i,1})
\;=\;
\det(P/P_1)\prod_{i=1}^{k}\det(\hat{P}_{i,1}).
\end{multline*}
where the last equality follows from~\eqref{eq.tail-system-matching}.
Combining this equation with $\det(P)=\det(P_1)\det(P/P_1)$ completes the proof of~\eqref{eq.schur-prod-expanded}. 
 
\end{document}